\newcommand{\klb}[1]{}
\newcommand{\ys}[1]{}
\newcommand{\longver}[1]{#1}
\providecommand{\longver}[1]{}
\longver{\newcommand{\shortver}[1]{}}
\providecommand{\shortver}[1]{#1}
\providecommand{\includeuncertainty}[1]{}

\documentclass{article} % For LaTeX2e
\usepackage[dvipsnames]{xcolor}
\longver{\usepackage[preprint]{neurips_2024}}
\shortver{\usepackage{neurips_2024}}
\providecommand{\withtime}[1]{#1}
\withtime{\usepackage{datetime2}}

\usepackage{amsmath,amsfonts,bm}

\def\eqref#1{equation~\ref{#1}}
\def\1{\bm{1}}

\def\ra{{\textnormal{a}}}

\DeclareMathAlphabet{\mathsfit}{\encodingdefault}{\sfdefault}{m}{sl}
\SetMathAlphabet{\mathsfit}{bold}{\encodingdefault}{\sfdefault}{bx}{n}

\usepackage[hidelinks]{hyperref}
\usepackage{url}
\usepackage{refcount}

\usepackage{authblk}
\usepackage{lipsum}
\usepackage{amsthm}
\usepackage{amsmath}
\newtheorem{definition}{Definition}
\newtheorem{procedure}[definition]{Procedure}
\newtheorem{theorem}{Theorem}
\newtheorem{lemma}{Lemma}

\numberwithin{equation}{section}
\numberwithin{definition}{section}
\usepackage{epsfig}
\usepackage{subfigure}
\usepackage{float}
\usepackage{placeins}
\usepackage{graphicx}
\usepackage{authblk}
\usepackage{enumitem}
\usepackage{todonotes}
\usepackage{booktabs}
\usepackage{makecell}
\usepackage{xspace}
\usepackage{natbib}
\usepackage{enumitem}
\usepackage{dsfont}
\usepackage{endnotes}
\setlist[description]{style=unboxed}
\newcommand{\citeyearp}[1]{[\citeyear{#1}]}

\newcommand{\commentout}[1]{}
\renewcommand{\eqref}[1]{(\ref{#1})}

\providecommand{\klb}[1]{\textbf{\textcolor{ForestGreen}{[KLB: #1]}}\xspace}
\providecommand{\ys}[1]{\textbf{\textcolor{CadetBlue}{[YS: #1]}}\xspace}

\newcommand{\denselist}{\itemsep 0pt\topsep-6pt\partopsep-6pt}

\title{Understanding Understanding: \\[.4em] \Large A Pragmatic Framework Motivated by Large Language Models}

\author[1,3]{Kevin Leyton-Brown}
\author[1,2]{Yoav Shoham} 
\affil[1]{{AI}21 {L}abs}
\affil[2]{Stanford University}
\affil[3]{University of British Columbia}

\newcommand{\deltaPG}{\delta_{\text{s}}}
\newcommand{\deltaOR}{\delta_{\text{r}}}

\newcommand{\p}{x}
\newcommand{\q}{y}

\begin{document}

\maketitle

\begin{abstract}
Motivated by the rapid ascent of Large Language Models (LLMs) and debates about the extent to which they possess human-level qualities, we propose a framework for testing whether any agent (be it a machine or a human) understands a subject matter. 
%The framework is agnostic about how the machine is built, in particular whether or not it is implemented as an LLM. 
In Turing-test fashion, the framework is based solely on the agent's performance, and specifically on how well it answers questions. Elements of the framework include circumscribing the set of questions (the ``scope of understanding''), requiring general competence (``passing grade''), avoiding ``ridiculous answers'', but still allowing wrong and ``I don't know'' answers to some questions. 
%All of these notions are defined precisely. 
Reaching certainty about these conditions requires exhaustive testing of the questions which is impossible for nontrivial scopes, but we show how high confidence can be achieved via random sampling and the application of probabilistic confidence bounds.  
%We  show that achieving acceptable bounds may require a very large number of samples, especially when it comes to the non-ridiculousness condition. However, 
We also show that accompanying answers with 
explanations can improve the sample complexity required to achieve acceptable bounds, because an explanation of an answer implies the ability to answer many similar questions.
% (this also provides intuition for why explanations play a key role in educational settings). 
According to our framework, current LLMs cannot be said to understand nontrivial domains, but as the framework provides a practical recipe for testing understanding, it thus also constitutes a tool for building AI agents that do  understand.
\end{abstract}

\section{Introduction}

The question of whether machines can understand is as old as
\longver{%
computing. Although Lady Ada Lovelace didn't explicitly address this question, her remarks on what a century later came to be called computers \citep{Lovelace} suggest that she would have answered negatively. Much later, \citet{SearleChinese} certainly did answer negatively with his Chinese Room argument. 
}%
\shortver{computing \citep[see, e.g.,][]{Lovelace,SearleChinese}. }%
The issue has now returned with renewed force in the context of Large Language Models (LLMs), with people questioning whether LLMs
really understand their subject matter. To some, modern LLMs are so advanced so that they can be said to be approaching human-like intelligence, or at least show sparks thereof \citep{Sparks}. To others, LLMs may be impressive and useful artifacts, but inherently cannot understand, and only mimic it superficially \citep{BenderKoller,Chomsky}.
These latter sentiments tend to based on arguments such that the systems engage in mere statistical calculations rather than reasoning; they deal with syntax but lack semantics; they lack a ``world model''; they lack ``communicative intent''; they are  information-poor relative to the rich sensory information available to a person embedded in the real world; they lack emotions, ``qualia'' or consciousness; and so on. 
(We discuss these skeptical views, as well as ones more sympathetic to LLMs, in Section~\ref{RelatedWork}.)

It's not hard to sympathize with the intuitions underlying the skeptical arguments, and even agree with some of the conclusions (e.g., that there is much that current LLMs can't be said to understand), but we find the arguments themselves weak for two reasons. First, they are vague and appeal to ill-defined concepts. For example, what would constitute satisfactory semantics? Would a mapping from linguistic constructs to any knowledge graph \citep{KG} do the trick? And second, we find these claims to be assumptive assertions without explanatory power. You assert that a world model is needed; I assert not; can we go have lunch now?

The underlying problem is that the discussion has lacked a clear definition of understanding.
Providing a mathematical definition of understanding is this paper's first technical contribution (see Section~\ref{sec:definition}). 
We take an approach akin to the Turing Test \citep{Turing}, which identified intelligence with behavior indistinguishable from human in a question-answering task. Roughly, we define a domain by a set of questions, and say that understanding the domain consists of being able to answer these questions in a satisfactory way. ``Satisfactory'' has two aspects. The first is demonstrating \emph{general competence}: 
assigning a score to each answer, we require that the average score across all questions be above a given, relatively high (for example, .7) threshold (the ``passing grade'').
The second aspect is \emph{avoiding ridiculousness}, which is defined in two steps. First, we define a ridiculous answer as one receiving a particularly low score (here we fix it at 0). Then, we require that the probability of a ridiculous answer be negligible (for example, below .001).  
Depending on the domain, answering ``I don't know'' can be acceptable: often not as good as providing a substantive answer, indeed sometimes ridiculous, but sometimes (if infrequently) the perfect answer. 

\longver{%
It might seem that a question-answering framework can be used only to \emph{assess} understanding, but we argue that it is also a way to \emph{define} the concept. In doing so, we  appeal to the philosophical tradition of Pragmatism, which argues that any concept's meaning can derive solely from its observable consequences. 
%In our setting this means that an agent should be said to understand if and only if it \emph{behaves} as an understanding agent would behave. 
For more discussion, see Section~\ref{RelatedWork}.
%Thus, (as with some objections to the Turing Test), it is nonsensical to seek some deeper notion of what understanding ``actually is'', perhaps based on the internal structure of the machine, with question answering being merely a way of testing for it empirically. If an agent never exhibits behavior that is inconsistent with understanding, it is meaningless to claim that the agent does not understand; 
%our everyday claims that agents (other than ourselves) understand are based only on behaviors rather than unobservable inner workings (consciousness, qualia, intent, etc).
}%

Our paper's second key contribution (in Section~\ref{sec:assessment-samples}) is to 
\shortver{%
    see how these criteria can be applied in practice. Asking all questions within scope--which is the only way to verify our criteria with certainty--is infeasible, so instead we provide probabilistic guarantees as a function of the number of questions sampled, leveraging concentration inequalities.
}%
\longver{%
    answer the following question: How can we ever realistically conclude that any agent (be it another person or an LLM) understands anything? Although our definition of understanding implies an obvious way of assessing whether an agent understands---namely, asking every question in the scope---all scopes of interest are so large (indeed, often infinite) that it is not possible in practice to verify our criteria with certainty. The natural alternative is to forego certainty and settle for obtaining high confidence, by asking a reasonable number of questions taken at random from the scope. We make these ideas formal too, appealing to concentration inequalities to investigate the sample complexity of testing understanding.

}%
Unfortunately, this analysis shows that the number of samples required to reach acceptable confidence bounds can be high, especially due to the non-ridiculousness criterion. Our third key contribution is to show how \emph{explanations} can mitigate the problem (see Section~\ref{sec:assessment-explanations}). In a nutshell, when an answer is coupled with an explanation, it effectively provides answers to a large number of similar questions one could have asked but didn't. We define this formally and show that it can meaningfully reduce the number of required samples. (This also provides intuition for why explanations play a key role in educational settings.) 

\longver{%
    A final introductory remark on ``why care''. One sometimes hears that AI practitioners interested in building intelligent machines needn't waste time on idle philosophizing if the agent does a good enough job. A similar sentiment used sometimes to be expressed in AI about the precise definition of causality (see \citet{PearlCausality} for a comprehensive treatment).
    In the case of causality the issue is pretty much settled by now, and we believe here too practitioners should care. As will be seen, the definition of understanding here is directly related to shortcomings of current LLMs, including being ridiculous when they're not brilliant, not being able to say ``I don't know'', and not providing explanations of their output. The definition of understanding here not only sheds light on these shortcomings, but also points to ways in which to avoid them; as we discuss in the last section, our framework has practical implications for both evaluating and designing AI agents. (This is beside the fact that, in general, a little philosophizing won't hurt us engineers, particularly in this age in which AI is impacting society so profoundly.) 
}%

\section{Related work}
\label{RelatedWork}

The topic of understanding is broad, and writing on it spans areas as diverse as philosophy, education, religion, and more. It is impossible to discuss all this work, but we provide some highlights.

As was mentioned in the introduction, this paper aims not only to explain how to assess understanding, but to also define the concept. Our approach falls in the camp of \emph{pragmatism} within philosophy---a rich and complex tradition, whose relevant part is summarized in  the so-called Pragmatic Maxim of \citet{Peirce1878}: 
\longver{\begin{quote}}\shortver{``}Consider what effects, that might conceivably have practical bearings, we conceive the object of our conception to have. Then, our conception of these effects is the whole of our conception of the object.\shortver{''}
\longver{\end{quote}}
In other words, for a concept to have meaning, it must have practical consequences, and if two concepts have the same practical consequences, then they are effectively the same concept.
According this principle, seeking some deeper notion of what understanding ``actually is'', such as by probing the internal state of the agent or its other implementation details, provides no useful information beyond the observed agent's competence in answering questions.
%\footnote{
%		The interested reader may find a second quote from \citet{Peirce1878} also useful, discussing the meaning of the concept of (physical) force. After describing how an object's velocity can be calculated as the vector sum between its previous velocity and that of a force applied to it, he writes: ``In how many profound treatises is not force spoken of as a ``mysterious entity,'' which seems to be only a way of confessing that the author despairs of ever getting a clear notion of what the word means! In a recent admired work on Analytic Mechanics it is stated that we understand precisely the effect of force, but what force itself is we do not understand! This is simply a self-contradiction. The idea which the word force excites in our minds has no other function than to affect our actions, and these actions can have no reference to force otherwise than through its effects. Consequently, if we know what the effects of force are, we are acquainted with every fact which is implied in saying that a force exists, and there is nothing more to know.'' We invite the reader to read the quote again, substituting ``understanding'' for each occurrence of ``force'' (and, say,  ``psychology'' for ``Analytic Mechanics''); we contend that the argument remains convincing.}
\longver{%
    If all behavioral evidence aligns with the claim that an agent understands, we should accept the claim.
}%
Our approach is  clearly at odds with that of \citet{PottsUnderstanding}%
\shortver{%
, who prefer \emph{internalism} (inspecting the internals of the system against some notion of what it means to understand) and \emph{referentialism} (mapping objects in the system to other, externally-defined concepts) to pragmatism. 
}%
\longver{%
. They point to three metaphysical (or ontological) stances: \emph{internalism} (inspecting the internals of the system against some notion of what it means to understand), \emph{referentialism} (mapping objects in the system to other, externally-defined concepts), and pragmatism. 
}%
They critique the pragmatic approach, arguing that tests (be they the Turing Test or various benchmarks in the AI community) are fallible, susceptible to manipulation and over-fitting, and find the other two approaches to be more definitive. 
\longver{%
    We clearly take the opposite position. 
}%
We find that although  internalism and referentialism capture valid intuitions, they are too ill-defined to be applied rigorously in any setting.
We do agree that the empirical tests mentioned and critiqued by \citet{PottsUnderstanding} are not the right ones for assessing  understanding, and suggest criteria that are. For a broader discussion within philosophy of understanding, with an emphasis on its distinction from `knowing' and a strong epistemological bent, see \citet{sep-understanding}. 

\longver{%
    We're not aware of previous work that introduced a rigorous theory of understanding---pragmatic or otherwise---that could be applied to answer the question of whether LLMs understand. However, there is certainly prior relevant (if mostly informal) writing that deserves discussion. 
    This includes two pieces that share the primary title of this paper \citep{UUWiggins,UUWurman}. Drawing from education theory, \citet{UUWiggins} focus
}%
\shortver{%
    Within the education literature there is also much (and mostly informal) writing on the topic. Perhaps the best exemplar is a piece that shares the primary title of this paper \citep{UUWiggins}. It focuses 
}%
on the educator's goal of  achieving in the learner true understanding versus memorized, rote knowledge; or, in their words, ``…safeguard the distinction between `understand' and `know the right answer when prompted'{''} (p.11). To them, the distinction hinges on being able to relate different concepts, to apply the knowledge in a broad set of circumstances, and to supply explanations. Three of their quotes from the literature capture these intuitions well. 
\shortver{First, from \citet{Dewey1933}: ``}\longver{\begin{quote}}%
To grasp the meaning of a thing~[…] is to see it in its relations to other things:~[…] what consequences follow from it, what causes it, what uses it can be put to~[…] Meanings are general because [they are] applicable in a wide variety of different instances in spite of their difference.\shortver{''}\longver{\citep{Dewey1933}\vspace{.5em}

}\shortver{ Second, from \citet{Bloom1956}: ``}Students should not be able to solve the new problems merely by remembering the solution to or the precise method of solving a similar problem in class.~[…] Understanding is the ability to marshal skills wisely and appropriately~[…] Doing something correctly, therefore, is not, by itself, evidence of understanding.~[…] To understand is often being reflected in being able to explain \emph{why}…\shortver{''}\longver{ \citep{Bloom1956}\vspace{.5em}

}\shortver{ Finally, from \citet{Bransford2000}: ``}Teaching specific topics or skills without~[…] a broader fundamental structure within a field of knowledge is uneconomical.\shortver{'' }\longver{\citep{Bransford2000}
\end{quote}
}Our framework can be thought of as concretizing and formalizing some of these intuitions. 

\longver{%
    \citet{UUWurman} is different---an epic, associative collection of evocative short essays on the nature of knowledge, understanding, learning, and other aspects of the human experience. It's not directly relevant to the current discussion, and we reference it here primarily because of the common title. But it does contain some relevant intuitions, for example in this quote:
    \begin{quote}
    The most essential prerequisite to understanding is to be able to admit when you don't understand something. (p.15)
    \end{quote} 
    As we'll see, these thoughts are directly related to our insisting on defining the scope of understanding and the inclusion of the ``I don't know" answer in our framework.
}%

Within AI proper there has also been much writing relevant to the topic. Some researchers lean towards ascribing a measure of understanding even to current LLMs. This includes \citet{Sparks}, who argue on experimental grounds that GPT-4 shows ``sparks of artificial general intelligence (AGI)'' 
(ill-defined as the term is).
Similarly, in a dialogue between two prominent researchers \citep{HintonNg}, both seem to agree that modern LLMs exhibit a measure of understanding, encoded in neural networks' internal structures. \citet{ArcasDaedalus} explicitly argues that LLMs should cause us to reevaluate our notion of understanding, and that ``statistics do amount to understanding, in any falsifiable sense''. 
\longver{%

}%
Most researchers are more circumspect on the topic, offering varying degrees of skepticism. Some, such as \citet{BenderKoller}, are completely dismissive of the idea that LLMs can exhibit any understanding at all. Somewhat less opinionated but still quite skeptical, \citet{MarcusDavisMITTR20, MarcusDavisCACM23} meticulously chronicle telltale mistakes of LLMs that, to the authors, indicate lack of understanding. 
Some authors point to specific technical elements that in their view are missing from LLMs. To pick two examples, \citet{RussinPFC} point to the higher-level functionality of the prefrontal cortex  and suggests that AI systems should learn from its structure by way of achieving out-of-distribution robustness; and \citet{XiaCausal} argue that learning in neural networks cannot reconstruct causal structures of the kind defined by \citet{PearlCausality}. 

Some of the most interesting writing on the topic presents new questions or perspectives, rather than offering highly opinionated answers. For example, \citet{WestEtAl} distinguishes between the ability to produce expert output and the ability to understand the output, and creatively points to how LLMs and humans may diverge in this regard.
%(in general, Choi has shared widely her nuanced thoughts on the limitations of modern LLMs, for example recently in \citep{ChoiCarroll}). 
%
\longver{

}%
Of particular note is a comprehensive survey of the AI literature on understanding by \citet{MitchellKrakauer}, which includes some of the above references and many others. This survey features clearly written commentary on the literature and the diverging views within it. The survey does not offer its own  definition of understanding or how it is to be evaluated, but rather concludes that there is a need for defining ``different modes of understanding'', with the possibility that understanding by machines will be different from that of humans. In this paper we propose a uniform framework for understanding understanding, which can be instantiated in different ways by varying the questions, answers, and explanations.
As \citet{MitchellKrakauer} note, the topic of whether machines can understand is still being actively debated. Our paper will not settle the debate, but hopefully will help move the conversation forward and put it on more solid ground.

\section{Defining Understanding}
\label{sec:definition}

To define understanding, we begin by observing  that it is never said simply that someone understands; there is always an explicit or implied subject matter that they do or don't understand. You understand arithmetic, you understand magnetism, you understand Buddhism, you understand human emotions. Let's call that which you understand the \emph{scope of understanding}, or simply the scope.
We define a scope by a set of questions with a distribution over them, and understanding by how well one is able to answer those questions (``well'' being defined below).\footnote{To anticipate a possible objection: One could imagine an argument that a person might understand a question perfectly but still not be able to answer it. But this is being unclear about the scope of understanding; there is no understanding without knowledge. If you can't solve any addition problem then you don't understand arithmetic, and if you can't answer any question about WWII then you don't understand world history.} Of course, a scope cannot be inferred by a term such as ``arithmetic''; we would judge whether a 10th grader understands arithmetic differently than whether a Field Medal winner understands it. So the scope will be defined by a process that generates the questions that comprise the scope. This process could be simply to sample from a distribution over an explicit set of questions, or a more involved procedure (e.g., draw two real numbers from [0, 1,000,000], draw an arithmetic operation from \{+,\,-,\,*,\,/\}) defining a large or even infinite set of questions.
%
%There are many additional aspects of scope that should be discussed. 
%
While we could define any set of questions as a scope, our framework is useful for sets that correspond to a conceptually natural domain, 
and natural domains are never small.
%But size isn't everything, and the set should be natural; for example, the set of arithmetic questions a 10th grader should be able to answer. 
Another issue is robustness to surface perturbations. If a question is in the scope, so should be its natural rephrasings, so that answers aren't overfit to specific surface formulations. There is the practical question of how to specify such a large set in practice, especially a set whose boundaries may be fuzzy. But here we take the scope as given. 

%Given a scope, understanding will be defined based on how well the agent answers the questions in the scope. As seen below, this will be defined by two criteria: General competence and non-ridiculousness. 

%With these intuitions in place, to the formal definitions.

\longver{%
    \begin{definition}
    A \emph{domain} is a tuple $(Q,\Delta_Q,A,S)$ where
    \begin{description}\denselist
    \item{--} $Q$ is a set of questions.
    \item{--} $\Delta_Q$ is a probability distribution over $Q$.
    \item{--} $A$ is a set of answers with $idk\in A$ ($idk$ standing for ``I don't know''). 
    \item{--} $S: Q\times A \rightarrow [0,1]$ is a scoring function for answers to given questions.
    \end{description}
    \end{definition}
}%
\shortver{%
    \begin{definition}
    A \emph{domain} is a tuple $(Q,\Delta_Q,A,S)$ where $Q$ is a set of questions; $\Delta_Q$ is a probability distribution over $Q$; $A$ is a set of answers with $idk\in A$ ($idk$ standing for ``I don't know''); and $S: Q\times A \rightarrow [0,1]$ is a scoring function for answers to given questions.
    \end{definition}
}%

When $S(q,a)=0$ we say that $a$ is a ridiculous answer to $q$.\footnote{
   	Observe that assigning the same score to all ridiculous answers prevents us from assessing degrees of ridiculousness. We made this modeling choice because our framework treats all ridiculous answers in the same way, and so we have no need for this additional flexibility. It would be straightforward to extend the framework to allow a range of ridiculous scores (e.g., all scores below 0.01) at the expense of additional complexity.} 
We do not constrain the values of $S(q,idk)$. Often such scores will be low; if you profess ignorance about the value of $1+1$ in the scope of arithmetic you get a very low score, indeed perhaps even a score of 0. But answering $idk$ to the question of whether $P=NP$ in the scope of computational complexity gives you a high score. That said, domains in which $idk$ gets high scores for all or most questions will rarely be interesting.\footnote{\shortver{However, one could}\longver{Rarely, but not never. For example, it could be interesting to} test an agent's understanding of the limits of its own knowledge via a scope consisting \longver{entirely }of facts that the agent has no way of knowing, such as \longver{current }events that occurred after the  training data cutoff date.}

\begin{definition}
\longver{%
    An \emph{understanding criterion} for a domain is a pair $(PG,RID)$ where
    \begin{description}\denselist
    \item{--} $PG\in (0,1)$ is the required level of overall competence, or the ``passing grade''.
    \item{--} $RID\in (0,1)$ where $RID << PG$ is a global ridiculousness threshold, the maximum allowable probability of obtaining a ridiculous answer when sampling a question from $\Delta_Q$.
    \end{description}
    }%
\shortver{%
    An \emph{understanding criterion} for a domain is a pair $(PG,RID)$ where $PG\in (0,1)$ is the required level of overall competence, or the ``passing grade''; and $RID\in (0,1)$ where $RID << PG$ is a global ridiculousness threshold, a maximum allowable probability of a ridiculous answer across all questions.
    }%
\end{definition}

The intuition should be that $PG$ is relatively high (for example, .7) and $RID$ is so low as to be negligible. 
One way of choosing a concrete value for $RID$ is to specify a test length $n_t$ on which the agent should be guaranteed to avoid making a ridiculous answer with high probability $1-\delta$. Then we can write
$ (1-RID)^{n_t} \geq 1-\delta.$
Solving for $RID$, we obtain 
% nt ln (1-rid) = ln (1-\delta)
% ln (1-rid) = ln (1-\delta) / n_t
% 1-rid = exp(ln (1-\delta) / n_t)
$RID = 1-\exp\left(\frac{\ln(1-\delta)}{n_t}\right).$
For example, if $\delta=0.05$ and $n_t=100$, $RID = 0.00052$; if $n_t = 1000$, $RID = 0.000052$. 

\begin{definition}
Given a domain and an understanding criterion as above, a question-answering (QA) system $ans: Q\rightarrow A$ is said to exhibit understanding within scope $Q$ if and only if:  
\begin{description}\denselist
\shortver{\vspace{-.2em}}%
\item{(Overall passing grade)} $\mathbb{E}_{q\sim \Delta_Q}\;S(q,ans(q)) \geq PG.$ 
%(i.e., the mean score according to $\Delta_Q$ must clear the passing grade threshold).
\item{(Ridiculous answers are rare)} $P_{q \sim \Delta_Q}(S(q,ans(q))=0) \leq RID$.
\end{description}
\label{def:understanding}
\end{definition}

%The definition allows the agent to make inadvertent ridiculous mistakes. But note that in settings in which ridiculous answers, if they exist, are not rare isolated cases, restricting the probability of a ridiculous answer to a sufficiently low value is tantamount to restricting it to zero.  \ys{WDYT?}
That is, the mean score according to $\Delta_Q$ must clear the threshold $PG$, and the probability with which the agent generates ridiculous answers must not exceed the (very low) limit $RID$.
Relative to this definition, and given reasonable values of $PG$ and $RID$, modern LLMs fall short. In very circumscribed scopes they do pass the criteria; for example, most do well when asked to identify the capital cities of countries. But they invariably fail in larger or more demanding scopes, witness the example in Appendix~\ref{TemporalExample}, which, empirically speaking, is not a rare isolated case (the example will also prove useful when we discuss explanations later in the paper). 

We make a few remarks on this last definition. First, observe that $\Delta_Q$ need not weight every question equally. This non-uniform distribution can reflect e.g., relative prevalence, importance, or difficulty of questions. Second, by their probabilistic nature,
LLMs can provide a probability distribution over answers (i.e., when we do not sample with temperature 0); one might therefore want to average over system executions, over  top-$k$ answers, or something else. The function $S$ can easily be modified to return an expected score and all of our results that follow carry through directly. 
Finally, it is possible to imagine a more draconian version of our ridiculousness condition which would set $RID =0$, that is, disallow any ridiculous answer, but there are several problems with this. First, no finite number of samples would suffice to distinguish between the cases where an agent never gives ridiculous answers and where it gives them with some infinitesimally small probability; thus, the draconian definition would preclude high-probability confidence bounds. Second, that definition is counterintuitive: we do conclude that other people understand when they perform well but very occasionally have idiosyncratic lapses. Finally, existing LLMs that exhibit ridiculous behavior do so more than a negligible fraction of the time, making the less stringent condition sufficient.

\section{Assessing Understanding via Independently Sampled Questions}
\label{sec:assessment-samples}

Given the framework just laid out, the only way to verify with certainty that an agent understands within a scope is to test it on all questions and see if the answers satisfy the conditions in Definition~\ref{def:understanding}. Even if the agent answers a large number of questions well, any questions we omit may reveal a lack of understanding; e.g., they might all yield ridiculous answers. 
Testing all questions is impractical for any scope of interesting size and  impossible for infinite-sized scopes. 
In such cases, however, it is possible to gain high \emph{confidence} that an agent understands by assessing its performance on many independent samples taken randomly from  the distribution of questions.  Intuitively, when we choose questions at random, the likelihood that the agent's performance is qualitatively different on the unasked questions as on the asked questions falls with the number of asked questions.

Our goal is to identify a test for understanding based on the agent's answers to independently sampled questions, and to prove that the probability that this test gives incorrect answers falls (ideally, quickly) with the number of questions asked. We can then investigate how many randomly chosen questions we need in order to obtain high confidence in the test's results. As we shall see, the number of questions required is large but not completely implausible. 
In the next section we will see how the notion of explanations can help to further reduce the required number of questions.

Formally, let $n$ be the number of questions we have asked, each sampled independently from $\Delta_Q$. Each question $q_i$ receives score $s_i =S(q_i,ans(q_i)) \in [0,1]$. 
Let $\hat{s}$ denote our sample estimate of the score $\frac{1}{n}\sum_{i = 1}^n s_i$ and let $s$ denote the true score $\mathbb{E}_{q\sim \Delta_Q}\;S(q,ans(q))$. 
Let $r_i = 1$ if $s_i = 0$ and $0$ otherwise.
Let $\hat{r}$ denote the fraction of the agent's answers that were ridiculous: $\frac{1}{n}\sum_i r_i$, and let $r$ denote the agent's true probability of giving a ridiculous answer, $P_{q \sim \Delta_q}(S(q,ans(q))=0)$. 

We need a few mathematical preliminaries. First, define $d(\p,\q) = \p \ln\left(\frac{\p}{\q}\right) + (1 - \p) \ln\left(\frac{1 - \p}{1 - \q}\right)$, replacing singularities ($\ln(0)$; division by zero) with limit values; see Definition~\ref{def:KL} in Appendix \ref{app:proofs}. This function can be interpreted as the relative entropy (or KL divergence) between two Bernoulli distributions\footnote{While $d(\p,\q)$ has an interpretation as the relative entropy between Bernoulli distributions, we can also simply consider it as an arbitrary function from $[0,1] \times [0,1] \rightarrow \mathds{R}^+$. We will use the functions $U$ and $L$, which depend on $d$, to upper bound the means of arbitrary (non-Bernoulli) distributions restricted to $[0,1]$.} having parameters $\p \in [0,1]$ and $\q \in [0,1]$.

Our testing procedure will appeal to two different functions. These do not have closed-form expressions, because $d$ does not have a closed-form inverse. Nevertheless, they are easy to approximate numerically because $d(\p, \cdot)$ is convex.
%k1
%\begin{equation}
%    B(n, \delta) = \sqrt{\frac{-\ln(\delta)}{2n}} \label{eq:B}
%\end{equation}
\begin{equation}
	U(\p, n, \delta) = \max\{\q \in [0, 1] : d(\p, \q) \leq \ln(1/\delta)/n\} \label{eq:U}
\end{equation}
\begin{equation}
	L(\p, n, \delta) = \min\{\q \in [0, 1] : d(\p, \q) \leq \ln(1/\delta)/n\} \label{eq:L}
\end{equation}

\newcommand{\itm}[1]{\item[\parbox{6em}{\textnormal{\emph{(#1)\hfill}}}]}
\newcommand{\GG}{\textit{Good-Grade}\xspace}
\newcommand{\BG}{\textit{Bad-Grade}\xspace}
\newcommand{\GR}{\textit{Good-Rid}\xspace}
\newcommand{\BR}{\textit{Bad-Rid}\xspace}
\newcommand{\GRp}{\textit{Good-Rid}$^\prime$\xspace}

We can now state our procedure for testing whether an agent understands a given scope.

\begin{procedure}[Testing Procedure Based on Independent Samples]
This procedure takes two arguments: a number of samples $n$ and a desired maximum failure probability $\delta$. %Let $\deltaPG = \deltaOR = \delta/2$.
\begin{enumerate}
	\item Determine that the agent understands the scope of $Q$ if:
	\begin{description}
		\itm{\GG} $L(\hat{s},n,\delta) \geq PG$; and
		\itm{\GR} $U(\hat{r},n,\delta) \leq RID.$ 
	\end{description}
	\item Determine that the agent does not understand the scope of $Q$ if:
	\begin{description}
		\itm{\BG} $U(\hat{s},n,\delta/2) < PG$; or
		\itm{\BR} $L(\hat{r},n,\delta/2) > RID$.
	\end{description}
\label{fail-PG}
	\item Otherwise, draw no conclusion; a larger $n$ (or larger $\delta$) is required.
\end{enumerate}
\label{def:testing-samples}
\end{procedure}
\vspace{1em}

\begin{theorem}
If Procedure~\ref{def:testing-samples} determines that an agent does or does not understand a given scope, this conclusion is correct with probability at least $1-\delta$.
\label{thm:samples}
\end{theorem}
\longver{%
    \vspace{-.5em}
    \begin{proof}
    See Appendix~\ref{app:proofs}.
    \end{proof}
}%

\noindent\emph{Remarks.} \shortver{All proofs appear in Appendix~\ref{app:proofs}.} Observe that our result requires no assumptions about the number of questions in the scope or about the probability distribution $\Delta_Q$; all that is needed is that questions are drawn independently. Note that \GG and \GR depend on $\delta$ but that \BG and \BR depend on $\delta/2$. This asymmetry arises because Step 2 asks whether either of two tests is true and so has two ways of failing, thus requiring multiple hypothesis correction. Step 1 asks whether \emph{both} of two tests are true, and hence does not require correction. Finally, in Appendix~\ref{app:Hoeffding-Chernoff} we explain why we leveraged the Chernoff bound, despite its lack of a closed-form expression, rather than the Hoeffding bound. In brief, while both bounds apply to our setting, Chernoff is much tighter when the true mean is close to 0 or 1.

Does this result enable the high-probability assessment of understanding within a reasonable number of questions? Let us consider some concrete numbers. Suppose that we are willing to tolerate an overall failure probability of $0.05$. First, let us consider average scores that would be required for \GG and \BG to hold as we vary the number of questions asked, summarized in Table~\ref{table:bounding-s-Chernoff}. 
Observe that our bounds depend on the observed empirical distribution, with confidence intervals narrowing as $|0.5 - \hat{s}|$ grows.
If the agent answers 100 questions and achieves a score of 0.9, it will pass if $PG < .811$ or fail if $PG > .963$; if it only achieves a score of 0.5, these thresholds change to passing if $PG < 0.379$ and failing if $PG > 0.634$, a gap of $0.25$ vs $0.15$. If it answers 10,000 questions and achieves a score of 0.9, it passes if $PG < 0.892$ and fails if $PG > 0.908$. 

Now let us consider the tests \GR and \BR and the value $RID = 0.00052$ derived in Section~\ref{sec:definition}; see Table~\ref{table:bounding-r-Chernoff}. 
Giving no 
ridiculous answers after answering 1,000 questions is insufficient to satisfy \GR with high probability ($U(0,n,0.05) \approx 0.0029913 > 0.00052$) but after it answers 10,000 questions without giving any ridiculous answers, it does satisfy \GR. If the agent first gives 3 ridiculous answers and thereafter always gives non-ridiculous answers, our upper bound starts out much higher but converges (more slowly) towards $0$ as $n \rightarrow \infty$. 
In this case \GR no longer holds after 10,000 questions but does after 100,000. If the agent consistently produces a 1\% empirical fraction of ridiculous answers, \GR will never hold ($U(\p,n,\delta) \geq \p$ for all $(n, \delta)$); \BR will not hold after 100 questions but will after 1,000 questions. 

Overall, a good rule of thumb seems to be that 1,000 questions are often enough to assess average score, and that 10,000 questions are often appropriate for assessing the probability of giving ridiculous answers. If we are interested in more tightly bounding the likelihood of observing ridiculous answers, every order-of-magnitude increase in the number of questions for which we would like the agent to be guaranteed with high probability to give no ridiculous answer roughly corresponds to an order-of-magnitude decrease in $RID$, and hence (referring to the table) roughly to an order-of-magnitude increase in the required $n$. Are these reasonable numbers in practice? Certainly for human agents they are not; no school test can have this many questions. One could argue that they are more plausible for AI systems, but this too depends on the setting. If one assumes that  the tests can be done in a completely self-supervised manner then perhaps they are, but this is a tenuous assumption. When a human assessment (of answer quality in general, and ridiculousness in particular) is needed, which arguably is the common case, then while the number of questions can be greater than when a person is the one answering them, it is still limited. The current cost of human annotation ranges between \$5--50 per item; at \$10 per assessment this brings the cost to six figures. Given that the process must be repeated as the set of questions evolves and as the model is retrained, this becomes prohibitive.
%Given that this would have to be done repeatedly as both the question-set and the LLM change, the cost can become prohibitive.  

\begin{table}[t] % this may end up going in an appendix anyway
	%	\begin{minipage}{.49\textwidth}\centering
		\centering
		\begin{tabular}{rcccccc}
			\toprule
$n$ & $L(0.9,n,0.05)$ & $U(0.9,n,0.025)$ & $L(0.5,n,0.05)$ & $U(0.5,n,0.025)$ \\
\midrule
10 & 0.545253 & 0.999023 & 0.164322 & 0.861187 \\
100 & 0.811171 & 0.962052 & 0.379423 & 0.633343 \\
1,000 & 0.875192 & 0.923796 & 0.461356 & 0.542868 \\
10,000 & 0.892497 & 0.907952 & 0.487763 & 0.513579 \\
100,000 & 0.897662 & 0.902557 & 0.496130 & 0.504295 \\
1,000,000 & 0.899264 & 0.900813 & 0.498776 & 0.501358 \\
			\bottomrule
		\end{tabular}
		%	\end{minipage}
	\vspace{.6em}	
	\caption{Example values for bounding ${s}$ in Procedure~\ref{def:testing-samples}, $\delta = 0.05$. \GG requires that the empirical average score exceed $PG$ by at least $L(\hat{s},n,\delta)$, and conversely \BG requires that $PG$ exceed the empirical average score by at least $U(\hat{s},n,\delta/2)$. The two pairs of columns thus represent the confidence intervals around empirical averages of $\hat{s} = 0.9$ and $\hat{s} = 0.5$ respectively. 
	}
	\label{table:bounding-s-Chernoff}
\end{table}

\begin{table}[t] % this may end up going in an appendix anyway
	%	\begin{minipage}{.49\textwidth}\centering
		\centering
		\begin{tabular}{rcccccc}
			\toprule
$n$ & $U(0,n,0.05)$ & $U(3/n,n,0.05)$ & $L(0.01,n,0.025)$ \\
\midrule
10 & 0.2588656 & 0.6783535 & 0.0000000 \\
100 & 0.0295130 & 0.0913315 & 0.0000933 \\
1,000 & 0.0029912 & 0.0094020 & 0.0036846 \\
10,000 & 0.0002995 & 0.0009429 & 0.0075333 \\
100,000 & 0.0000300 & 0.0000943 & 0.0091693 \\
1,000,000 & 0.0000030 & 0.0000094 & 0.0097321 \\
			\bottomrule
		\end{tabular}
		%	\end{minipage}
	\vspace{.6em}	
	\caption{Example values for bounding $r$ in Procedure~\ref{def:testing-samples} with $\delta = 0.05$. If the empirical fraction of ridiculous answers is zero, \GR holds if $U(0,n,0.05) \leq RID$. In the case that exactly three ridiculous answers are observed independently of $n$, \GR holds if $U(3/n, n, 0.05) \leq RID$. In the case that $\hat{r} = 0.01$ for every $n$, \BR holds if $L(0.001,n,0.025) > RID$.
	}
	\label{table:bounding-r-Chernoff}
\end{table}

%k2
\newcommand{\proc}{\ensuremath{}\alpha}
\newcommand{\Proc}{\mathcal{A}}
\newcommand{\Qp}{\ensuremath{}Q_{\proc}}
\newcommand{\Qpi}{\ensuremath{}Q_{\proc_i}}
\newcommand{\DQpi}{\Delta_{\Qpi}}
\newcommand{\bA}{\ensuremath\bar{Q}}
\newcommand{\bAi}{\ensuremath\bar{Q_{\proc_i}}}
\newcommand{\pai}{p_{\proc_i}}
\newcommand{\sai}{s_{\proc_i}}
\newcommand{\rai}{r_{\proc_i}}
\newcommand{\pa}{p_{\proc}}
\newcommand{\sa}{s_{\proc}}
\renewcommand{\ra}{r_{\proc}}
\newcommand{\bS}{\bar{X}}
\newcommand{\bn}{\bar{n}}
\newcommand{\br}{\bar{r}}
\newcommand{\hbs}{\hat{\bar{s}}}
\newcommand{\hbr}{\hat{\bar{r}}}
\newcommand{\bp}{\bar{p}}
\newcommand{\deltaPGS}{\delta_{\text{PG-}\bS}}
\newcommand{\bs}{\bar{s}}
\newcommand{\piai}{{\pi}_{\proc_i}}
\newcommand{\hpiai}{\hat{\pi}_{\proc_i}}
\newcommand{\hpia}{\hat{\pi}_{\proc}}
\newcommand{\dai}{\frac{\delta}{|\Proc|+1}}
\newcommand{\daii}{\frac{\delta}{2(|\Proc|+1)}}
\newcommand{\nai}{n_{\proc_i}}
\newcommand{\na}{n_{\proc}}

\section{Explanations: Beyond Independent Sampling}
\label{sec:assessment-explanations}

We have seen that assessing understanding with acceptable confidence based only on answers to a sample of questions can require very large sample sizes, especially as we push for a small $RID$. We now show that when explanations\longver{\footnote{One could replace the term `explanation' with `justification', but we stick to the former since it's the term most commonly used in connection with LLMs.}} accompany answers, we can achieve better sample efficiency. 
The idea that assessing understanding is aided by eliciting explanations is quite familiar from educational settings, where
homework and exam questions often end with “explain your answer”. As noted in the (mostly informal)
education literature we surveyed in Section \ref{RelatedWork}, good explanations
demonstrate knowledge of general principles and the ability to apply them broadly. Explanations have also received attention in the NLP community; e.g., one line of work quantifies the value of explanations as the performance gain they offer a student model seeking to simulate a teacher \citep{pruthi2022evaluating}.

Explanations have many facets and come in different forms. A common form is the causal explanation, for example as defined by \citet{PearlWhy}, who define  three increasingly sophisticated explanatory levels (the ``ladder of causation”). Another facet of causal explanations is that they can come at multiple granularities, as demonstrated in Feynman's \citeyearp{FeynmanWhy} beautiful response to the lay question of ``why magnets attract”. But causal explanations, while common and important, are a special case. For example, when we ask a student to explain their answer to a mathematical problem, we don't expect a causal explanation, but rather the steps in the derivation, or at least mention of the mathematical principles applied. In other cases, such as in response to factual questions (for example, ``Is John an employee of Acme Corp?"), a good explanation is a reference to a definitive source of relevant information (for example, Acme's database of employees). \longver{Another source of authoritative information can be a a trusted person.\footnote{One is reminded of Angluin’s \citeyearp{AngluinProofs} tongue-in-cheek list of bogus proofs techniques, which
includes the category of ‘proof by eminent authority’ (``I saw Karp in the elevator and he said it was
probably NP-complete").}}
What is common to all of types of explanation is that they involve a \emph{procedure} that is deemed authoritative and that could similarly be applied to many other questions. As a result, the agent can be given some amount of credit for those questions without these being asked explicitly; ``an explanation is worth a thousand questions". 
%In the following sections we make this intuition precise, and explore its ramifications. 

%\subsection{Setting}

Formally, we augment our existing model so that, when we ask the agent a question, the agent sometimes couples its answer with a particular procedure that it used to derive the answer. When it does, we get evidence about how the agent would have answered similar questions.  
\longver{It's important that this be an unambiguous procedure.}% 
For example, we might prompt the LLM to use \emph{chain-of-code} reasoning \citep{li2023chain}, explicitly asking the LLM to produce code to answer the question. The explanation could also be a mathematical equation, a precise table being looked up, and so on. But it cannot be any string of words that purport to constitute an explanation; Appendix~\ref{TemporalExample} provides an example of an LLM offering convincing-sounding explanations that aren't worth the metaphorical paper they're written on.  

Formally, let $\Proc = \{\proc_1, \ldots, \proc_k\}$ be a set of procedures offered by the agent to explain how it derived answers to a corresponding set of questions. Each procedure $\proc$ is potentially applicable to more than just the single question that prompted it. We assume that set of questions to which each $\proc_i$ is applicable is known, denoting this set $\Qpi \subseteq Q$, and that $\forall i,j\; \Qpi \cap\, Q_{\proc_j} = \emptyset$.
We further assume that the answer generated by each explanation $\proc_i$ achieves the same score on all questions to which it applies;\footnote{Our results can be generalized beyond this case at the cost of notational complexity.} denote this score $\sai$. If $\sai=0$, the answers generated by the explanation are ridiculous for every $q_i \in \Qpi$, which we denote $\rai = 1$; otherwise $\rai = 0$. We define the set of ``remaining'' questions to which no procedure applies as $\bA$ 
%\ys{at least on my screen this is indistinguishable from $Q$}\klb{Looks pretty different for me, but I'm happy to change it if you think of something better. Note that I need to use the same modifier on all kinds of variables, so bar was a natural choice.} 
and thus write $Q = \left(\bigcup_{\proc_i \in \Proc} \Qpi\right) \,\cup\, \bA$. Let $\pai = P_{q \sim \Delta_Q}(q \in \Qpi)$, the probability that a random question from the original distribution $\Delta_Q$ will belong the set of questions covered by explanation $\proc_i$. 
Let $\DQpi$ denote the renormalized restriction of probability distribution $\Delta_Q$ to domain $\Qpi$: 
%\ys{Not simpler to say that $\DQpi(q_i)$ is the marginal of $\Delta_Q$ on $\Qpi$?}\klb{I figured we should formally write it out either way, but sure it's the same expression. I'm just not sure if we have the right notation to be writing conditional probability statements. How would you write it?} 
for all $q_i \in \Qpi$, $\DQpi(q_i) = \frac{1}{\sum_{q_j \in \Qpi} \Delta_Q(q_j)} \cdot \Delta_Q(q_i).$ 
%Then, $\pa = \sum_i \pai$. 

Given $\Proc$, what should we believe about the agent's performance on our original test $Q$? 
%If $\Qp = Q$, we now know how the agent would answer all questions, and we can say with certainty that the agent understands if $\sa \geq PG$ and $\ra \leq RID$ and that it does not understand otherwise. 
%We can do nearly as well whenever $\pa$ is so large that a test of length $n_t$ is extremely unlikely to contain any questions beyond $\Qp$: when $(\pa)^{n_t} \geq 1-\delta$ we can say that with probability at least $1-\delta$, the agent understands if $\sa \geq PG$ and $\ra \leq RID$, and likewise that it does not understand if if $\sa < PG$ or $\ra > RID$. 
%We are also able to draw a conclusion if the trivial bound on the score the agent would achieve on $\bA$ suffices: the agent certainly does not understand if $\pa \cdot s_a + (1-\pa) \cdot (1) < PG$ and certainly does understand if $\pa \cdot s_a + (1-\pa) \cdot (0) = \text{simplified expression} > PG$. 
%
Intuitively, while a procedure $\proc$ can tell us more than an answer to a single question, we do not obtain the set $\Qp$ via independent sampling, so we cannot directly incorporate it into our previous analysis.
Imagine that in addition to $\Proc$ we have a set $\bS$ of additional answers to questions sampled independently from $\bA$, denoting $|\bS|$ as $\bn$, and that the agent does not give an explanation for any question in $\bS$. In this section our key question will be: How much better off are we having received explanations $\Proc$ than in the situation where we did not receive these explanations and took $\bn$ samples from $Q$?
%Consider each non-explained sample $q_j \in S$. If $q_j \in A$, the sample offers no information beyond the procedure $a$, and so we can simply discard it. Otherwise, $q_j \in \bA$ and so it provides additional information; keep the sample. 

\includeuncertainty{\subsection{When Application of Trusted Procedures is Certain}}

In our analysis below we make two substantive assumptions: that we believe that the agent claiming to have used explanation $\proc$ in fact did, and (even more optimistically) that the agent would apply it reliably to every question in $\Qp$. Procedure~\ref{def:testing-explanations} tests an agent's understanding in this setting.

\newcommand{\GGE}{\mbox{\textit{Good-Grade-Exp}}\xspace}
\newcommand{\BGE}{\mbox{\textit{Bad-Grade-Exp}}\xspace}
\newcommand{\GRE}{\mbox{\textit{Good-Rid-Exp}}\xspace}
\newcommand{\BRE}{\mbox{\textit{Bad-Rid-Exp}}\xspace}

%k3
\begin{procedure}[Testing Procedure Based on Samples and Explanations]
This procedure takes two arguments: a number of samples $n$ and a desired maximum failure probability $\delta$.
\begin{enumerate}
	\item Determine that the agent understands the scope of $Q$ if:
	\begin{description}[labelwidth=2.8cm]
		\itm{\GGE} $\left(\sum_{\proc_i \in \Proc} \pai \cdot \sai \right) + \bp \cdot L(\hbs,\bn,\delta) \geq PG$; and \label{pass-PG2}
		\itm{\GRE} $\left(\sum_{\proc_i \in \Proc} \pai \cdot \rai \right) + \bp \cdot U(\hbr,\bn,\delta) \leq RID.$ \label{pass-rid2}
	\end{description}
	\item Determine that the agent does not understand the scope of $Q$ if:
	\begin{description}[labelwidth=2.8cm]
		\itm{\BGE} $\left(\sum_{\proc_i \in \Proc} \pai \cdot \sai \right) + \bp \cdot U(\hbs,\bn,\delta/2) < PG$; or
		\itm{\BRE} $\left(\sum_{\proc_i \in \Proc} \pai \cdot \sai \right) + \bp \cdot L(\hbr,\bn,\delta/2) > RID$.
	\end{description}
	\label{fail-PG2}
	\item Otherwise, draw no conclusion; a larger $n$ (or larger $\delta$) is required.
\end{enumerate}
\label{def:testing-explanations}
\end{procedure}

\begin{theorem}
	If Procedure~\ref{def:testing-explanations} determines that an agent does or does not understand a given scope, this conclusion is correct with probability at least $1-\delta$.
	\label{thm:explanations}
\end{theorem}
\longver{%
    \vspace{-1.5em}
    \begin{proof}
    	See Appendix~\ref{app:proofs}.
    \end{proof}
}%

Respectively comparing each of our conditions (\GGE vs \GG, etc) answers the question we asked earlier about the quantitative power of explanations: our confidence bounds shrink in proportion to the fraction of $Q$ covered by procedures in $\Proc$. Stated another way, we would need to take between a factor of 
$1/\left(1-\sum_{\proc_i \in \Proc}\pai\right)$ and
$1/\left(1-\sum_{\proc_i \in \Proc}\pai\right)^2$
more samples to achieve the same tightening of our confidence bounds, with the rate of tightening depending on the true distributions $\hbs$ and $\hbr$.
%For \GGE and \BGE, the size of the effect depends on $PG$. When The greatest effect occurs for $PG=0.5$, in which case we would need to take a factor of $1/\left(1-\sum_{\proc_i \in \Proc}\pai\right)^2$ more samples from $Q$ without having received any explanations. The sample complexity of our bounds improves as the true mean approaches $0$ or $1$, so for \GRE and \BRE, the bounds scale nearly linearly when $RID$ is close to zero (as illustrated in Table~\ref{table:epsilon}). In this case we receive roughly the same effect as taking a factor of $1/\left(1-\sum_{\proc_i \in \Proc}\pai\right)$ more samples from $Q$ without having received any explanations.
For example, imagine that $Q$ contains an equal number of problems about multiplication and about literature, that the agent gives us its procedure $\proc$ for performing multiplication, and that the agent is 33\% accurate at answering questions about literature. Thereafter it suffices to sample only literature questions and each counts in our bounds for average grade as much as $1/(1-0.5)^2 = 4$ questions from the original distribution and in our bounds for the probability of giving ridiculous answers about as much as $1/1-0.5 = 2$ questions, in both cases yielding a bound exactly half as wide. This dramatic improvement in the bound combines two effects. First, our knowledge of $\proc$ eliminates all uncertainty about $\sa$ and $\ra$; second, we devote all samples to $\bA$ rather than devoting a $p_a$ fraction of samples to $\Qp$.

Several extensions and refinements are possible. We have so far described a setting in which all explanations and answers are offered before our analysis occurs. We can straightforwardly construct an online procedure, with the agent offering either individual answers or explanations as it chooses and our bounds updating as we go. Before we have asked any questions, we begin with $\Qp = \emptyset$, $\bA = Q$, and  $\Proc = \emptyset$. We always sample only from $\bA$, storing these samples in $\bS$. When we receive an explanation $\proc_i$, we add it to $\Proc$, subtract the set covered by $\Qpi$ from  $\bA$, and discard any previous samples from $\bS$ that do not pertain to the newly restricted $\bA$ (allowing for the possibility that the agent may later offer an explanation that covers previous individual answers). Theorem \ref{thm:explanations} then applies directly at every point in this process. More demanding is relaxing the assumption that the agent would apply its supplied procedure to all other applicable questions. 
\shortver{We leave the consideration of alternate models to a longer version of this paper but note that the procedure we have introduced here constitutes a sample complexity lower bound on all such models.
}
\longver{
We discuss this in the next section.
}

\includeuncertainty{
\subsection{When Application of Trusted Procedures is Uncertain}
\label{sec:ExplanationUncertainty}

\klb{Here's my proposed new model. As before, the questions are partitioned according to where the procedure could be applied. Each explanation confers confidence that the procedure would be used for some fraction of the relevant questions. You'd proposed that this be either a constant or a constant fraction; I propose more generally that it be a fraction sampled IID from a distribution $f$ with support on $[0,1]$, with *which* questions these are always being chosen uniformly at random. The existing Section 5.1 is then the special case where $f$ is a point mass at 1. Consider first the special case where f is a point mass at some value other than 1; e.g., 0.5. Then the first explanation gives us expected coverage over 50\% of the questions; the second explanation again gives us coverage over 50\%, but 25\% are questions that were already covered by the first explanation and 25\% are new. So it gives us a new 1/4. Now 75\% are covered, so when we get a third explanation on expectation 3/4 * 1/2 of the weight goes to questions that were already covered and 1/4 * 1/2 = 1/8 goes to new questions. In general, let $c_i$ be the fraction of questions covered by the $i$th explanation; the coverage of the first explanation is $0.5$ and of the $i$th explanation is then $0.5*(1-\sum_{j \in \{0, \ldots, i-1\}} c_j)$. We can solve the general recurrence relation in closed form: if the coverage probability is $p_c$, the $i$th element is $p_c * (1-p_c)^{i-1}$.
Overall, this gives us a more principled story behind the harmonic series, with one parameter, $f$. If we change $f$ to a point mass at say 0.2, we'd still converge to probability 1 in the limit: the first explanation would cover 20\%, the second would cover 0.8*0.2 = 16\% new questions, for total coverage of 0.36; the third would cover 0.64*.2 = 12.8\%, for a total of 0.488, and so on. 
And clearly this could be extended to other distributions, where the coverage offered by each explanation is not the same. I wouldn't work through any such examples in the paper, but I think it's nice that the model directly covers such cases and it's easy to see how to work it through.}

Let us now consider a more general setting. As before, we assume that we know the procedure $\proc_i$ that the agent used for a given question $q_i$, but we no longer assume that we are certain that the agent will continue to use $\proc_i$ for the other questions to which it is applicable, $\Qpi$. 
%Let us assume in instead that the agent offers derivations---such as the trace provided in chain-of-thought reasoning \citep{wei2023chainofthought}---that allows us to check that $a_i$ was correctly applied to any given question $q_i \in A_i$.
%
Now we have a bounding problem just like our analysis of ridiculousness: the question of whether the agent will apply the procedure on a given, future question from $\Qpi$ is a Bernoulli random variable.
Let $n_{\proc_i}$ denote the number of questions we sample from $\Qpi$, let $\piai$ denote the (unknown) true probability with which the agent applies the procedure on questions from $\Qpi$, and let $\hpiai$ denote our empirical observation of the fraction of questions from $\Qpi$ for which the agent uses $\proc$.

\newcommand{\GGU}{\mbox{\textit{Good-Grade-EUA}}\xspace}
\newcommand{\BGU}{\mbox{\textit{Bad-Grade-EUA}}\xspace}
\newcommand{\GRU}{\mbox{\textit{Good-Rid-EUA}}\xspace}
\newcommand{\BRU}{\mbox{\textit{Bad-Rid-EUA}}\xspace}

\begin{procedure}[Testing Procedure Based on Samples and Explanations with Uncertain Application]
	This procedure takes one argument: a desired maximum failure probability $\delta$.
	\begin{enumerate}
		\item Determine that the agent understands the scope of $Q$ if:
		\begin{description}[labelwidth=2.8cm]
			\itm{\GGU} $\left( \sum_{\proc_i \in \Proc} \pa \cdot L\left(\hpiai,\nai,\dai\right) \cdot s_\alpha
			\right) + \bp \cdot L\left(\hbs,\bn,\dai\right) \geq PG$; and 
		\itm{\GR} $U\left(\hat{r},n,\delta\right) \leq RID.$ 
		\end{description}
		\item Determine that the agent does not understand the scope of $Q$ if:
		\begin{description}[labelwidth=2.8cm]
			\itm{\BGU} $\left( \sum_{\proc_i \in \Proc} \pa \cdot U\left(\hpiai,n_{\proc_i},\daii\right) \cdot s_\alpha\right.$\\ \phantom~~~~~~~~~~~~~~~~~~~~~~~~~~~~~~~~~~$
			\left. + \left(1-U\left(\hpiai,\nai,\daii\right)\right)
			\right) +\, \bp \cdot U\left(\hbs,\bn,\daii\right) \leq PG$; or 
		\itm{\BR} $L\left(\hat{r},n,\delta/2\right) > RID$.
		\end{description}
		\item Otherwise, draw no conclusion; a larger $n$ (or larger $\delta$) is required.
	\end{enumerate}
	\label{def:testing-explanations-uncertain}
\end{procedure}

\begin{theorem}
	If Procedure~\ref{def:testing-explanations-uncertain} determines that an agent does or does not understand a given scope, this conclusion is correct with probability at least $1-\delta$.
	\label{thm:explanations-uncertain}
\end{theorem}
\vspace{-1.5em}
\begin{proof}
	See Appendix~\ref{app:proofs}.
\end{proof}

Let us return to our multiplication and literature example, but now imagine that we are not certain that the agent will apply the correct procedure to all multiplication questions. Let us contrast the three tests we have developed for determining whether the agent's average grade exceeds PG: \GG, \GGE, and \GGU. Consider a case where the average grade is $0.6$ for both categories of tests and where the agent is always observed to follow the multiplication procedure when tested. Table~\ref{table:comparing-tests} shows a comparison of the three tests. \GGU is at a significant disadvantage in two senses: first, it calculates two lower bounds, each of which has half the failure probability of the single lower bounds conducted by \GGE; second, each \GGU's two lower bound also receives half the number of samples. Nevertheless, for all but the smallest sample sizes, \GGU yields a tighter bound overall in this case. Table~\ref{table:comparing-tests-bounds} shows the values of these four lower bounds that are used by the three tests.

\begin{table}[t]
	\centering
\begin{tabular}{rccc}
\toprule
$n$  & $\GG$ & $\GGU$ & $\GGE$ \\
\midrule
10  & 0.236717 & 0.194629 & 0.351177 \\
30  & 0.378821 & 0.382271 & 0.432801 \\
100  & 0.477904 & 0.483170 & 0.504507 \\
300  & 0.529782 & 0.537422 & 0.544709 \\
1,000  & 0.561737 & 0.567658 & 0.569863 \\
3,000  & 0.577984 & 0.581941 & 0.582678 \\
10,000  & 0.587970 & 0.591856 & 0.592961 \\
\bottomrule
\end{tabular}
\vspace{0.6em}
\caption{Comparing the three approaches for lower bounding the agent's average grade taken from the conditions \GG, \GGU, and \GGE. Throughout $s = s_{\proc} = \bs = 0.6$. All three tests were run with the same overall number of samples, $n$. For \GGU, $n_\proc = \bn = n/2$; for \GGE, $n_\proc = 1$ and so $\bn=n-1$. \GG does not reason about explanations so it is simply based on all samples. All tests are run with overall failure probability of $0.05$ (e.g., \GGU has two internal tests, and each is given a failure probability of $0.025$.)}
\label{table:comparing-tests}
\end{table}

\begin{table}[t]
\centering
\begin{tabular}{rcccc}
\toprule
$n$ & $L(0.6,n,0.05)$ & $L(1,n/2,0.025)$ & $L(0.6,n/2,0.025)$ & $L(0.6,n-1,0.05)$ \\
\midrule 
10 & 0.236717 & 0.478176 & 0.102353 & 0.220901 \\
30 & 0.378821 & 0.831567 & 0.265601 & 0.375194 \\
100 & 0.477904 & 0.928878 & 0.409013 & 0.477287 \\
300 & 0.529782 & 0.975707 & 0.489419 & 0.529663 \\
1,000 & 0.561737 & 0.992649 & 0.539726 & 0.561717 \\
3,000 & 0.577984 & 0.997544 & 0.565356 & 0.577981 \\
10,000 & 0.587970 & 0.996318 & 0.585922 & 0.587970 \\
\bottomrule
\end{tabular}
\vspace{.6em}
\caption{Values for the four lower bounds that are used by the three tests in Table~\ref{table:comparing-tests} can be calculated. Specifically, 
$\text{\GG} = L(0.6,n.0.05)$,
$\text{\GGU} = 0.5 \cdot L(1,n/2,0.025) \cdot 0.6 + 0.5 \cdot L(0.6,n/2,0.025)$; 
$\text{\GGE} = 0.5 * 0.6 + L(0.6, n-1, 0.05)$.}
\label{table:comparing-tests-bounds}
\end{table}

The reader may wonder why Procedure~\ref{def:testing-explanations-uncertain} leverages the conditions \GR and \BR from Procedure~\ref{def:testing-samples} rather than deriving new conditions that benefit from having access to explanations. In brief, the former offers tighter bounds than the latter; the issue is discussed in detail in Appendix~\ref{app:uncertainty-ridiculousness}.
}%

\section{Discussion}

This paper addressed the question of how one is to judge whether an agent such as an LLM  understands a subject matter. 
We note that our focus on question answering excludes other behaviors that could commonly be viewed as demonstrating understanding. For example, we could have said that the agent (or our dog) understands our commands based on whether it follows them. Or, we could have said that you show that you understand a person's sorrow with your kind look and comforting embrace. In general, people use language in flexible ways, and terms such as ‘understanding' are used in multiple senses. We focus on question answering since it is the closest in spirit to the discussion of whether computer systems such as LLMs understand, which has been the impetus for this paper. 

\longver{%
A note about the potential danger of omitting ``obvious'' questions. People have vast background---sometimes called ``common sense''---knowledge, which the person formulating the list of questions would not think of including since they are so obvious. There could be the concern that the AI agent might answer the more sophisticated questions well, but would have failed on these obvious questions had they been asked. We're not sure how serious this concern is, since arguably without this background knowledge the agent would eventually fail on the more sophisticated questions. But it is something to watch out for.
}%

The framework discussed here can be expanded in various ways. Some of these were discussed in previous sections, but there are others. For example,
\longver{%
rather than being compared to the expected score, the passing grade could be applied to another norm, such as the max. This is admittedly an edge case, but is applicable, for example, if all questions are open questions in the field; 
if you solved P/NP no further evidence is required to show that you understand the theory of computing.
Or, as another example, 
}%
the framework as discussed is a ``static'' one, in two senses. First, it asks whether an agent in a given state understands something. It does not discuss the dynamics of understanding, for example how an agent  can come to understand things, such as through a series of Socratic questions posed to it.  Second, it assumes a given scope, but one can imagine realistic situations in which an initially specified scope is seemed insufficient and is expanded as a result of interaction with the agent that uncovered new relevant questions that had not been anticipated. Both are interesting topics for further research. 

Still, already as laid out here, the framework is a good basis for discussing the relationship between LLMs and understanding, and the framework can be applied to both the analysis and design of AI agents. From the analysis point of view, the framework suggests that beside testing the AI agent against a fixed test suite such as HELM \citep{HELM} or MMLU \citep{MMLU}, the user of the agent should require that the agent  virtually never provide ridiculous responses. Furthermore, the ridiculous answers need not be entirely specified in advance, which could both be infeasible, and could invite overfitting, but rather can be presented or expanded post-hoc on an ongoing basis. To be sure, this places a burden on the agent builder, who is held accountable to unforeseen future tests, but that does not seem unfair if one is claiming understanding by the agent one is offering. From the design as well as the analysis point of view, the framework suggests two additional things. First, that the agent admit ignorance when it does not have the answer, provided pleading ignorance is not ridiculous. And second, when it does not plead ignorance, that  answers be complemented by explanations, which consist of rigorous, accepted procedures by which the answers were derived (be they a simple database lookup or a more complicated program). 

Clearly, according to the account of understanding laid out in this paper, current AI systems---in particular, LLMs---can't be said to understand nontrivial domains. The systems are unreliable, occasionally downright ridiculous, unpredictable, opaque, and non-explainable. This is not only of intellectual interest but is intimately tied to the practical limitations industry is experiencing in deploying them broadly. 

As we've mentioned, our approach is similar to that of \citet{Turing}, though the details are different; in this connection, see \citet{GroszAIM2012}. But Turing's seminal article is relevant in another way as well.  The jury is out on which set of methods holds the greatest promise in meeting the criteria for understanding laid out here. Some will surely believe all it would take is scale---bigger models, larger and better training data. Others will believe that the path forward involves the so-called neuro-symbolic approach, which combines the epically knowledgeable and creative yet stochastic LLMs with classical methods based on semantically clear data structures and human-interpretable algorithms. Time will tell, not only which approach triumphs, but also whether the goal is attainable at all. Can AI achieve human-level understanding, or will that be forever out of reach? Here we think it's instructive to recall the words of Turing in the same article:
\longver{\begin{quote}}\shortver{``}%
It is customary to offer a grain of comfort, in the form of a statement that some peculiarly human characteristic could never be imitated by a machine. I cannot offer any such comfort, for I believe that no such bounds can be set.%
\shortver{''}\longver{\end{quote}}
We believe this is true in particular of understanding. 

\clearpage
\longver{ 
\section*{Acknowledgements}
We thank Noga Alon for feedback on the probabilistic analyses;  Barak Lenz for pointing out Feynman's ``why magnets attract'' example; Moises Goldszmidt for the observation about the danger of omitting obvious examples; Scott Aaronson for the observation that solving P/NP would be sufficient to demonstrate understanding of the theory of computing; and Chris R\'e and attendants of the Stanford HAI colloquium \citep{HAIShoham}, in which some of the ideas in this paper were discussed, for raising the question of the dynamics of understanding. Many additional thanks to Ernie Davis, Mansi Gupta, Joe Halpern, Percy Liang, Judea Pearl and Danish Pruthi for their useful comments. 
}

\section*{Bibliography}
\bibliography{refs}
\bibliographystyle{plainnat}

\clearpage
\appendix

%\section{Endnotes}
%
%\theendnotes

\section{Deferred Technical Material}
\label{app:proofs}

\begin{definition}[Relative entropy between Bernoulli distributions]
	Let $d(\p,\q)$ denote the relative entropy between two Bernoulli distributions having parameters $\p \in [0,1]$ and $\q \in [0,1]$: 
	\[d(\p,\q) = \begin{cases}
		\p \ln\left(\frac{\p}{\q}\right) + (1 - \p) \ln\left(\frac{1 - \p}{1 - \q}\right)  &  \p \in (0,1) \text{~and~} \q \in (0,1) \\
		\ln(1/(1 - \q))  &  \p = 0 \text{~and~} \q \neq 1 \\
		\ln(1/\q) & \p = 1 \text{~and~} \q \neq 0 \\
		%	0 & \p = 0 \text{~and~} \q = 0 \\
		\infty & \p \neq 0 \text{~and~} \q \in \{ 0, 1\} \\
		\infty & \p \in \{0,1\} \text{~and~} \q = 1-\p \\
		%	0 & \p = 1 \text{~and~} \q = 1 \\
		%	\infty & \p \neq 0 \text{~and~} \q = 1. 
	\end{cases}\]
	\label{def:KL}
\end{definition}

Our proof of Theorem~\ref{thm:samples} relies on five lemmas.
We begin with a Chernoff bound from the literature, following \citet{lattimore2020bandit} but using notation aligned with the rest of our paper and stating a slightly stronger version of the result than is presented in their Corollary 10.4: specifically,  restricting only to bounded random variables rather than to Bernoulli random variables. \citeauthor{lattimore2020bandit} explain why this stronger version of their Corollary 10.4 also holds in their Notes 10.1 and 10.2 on page 140, leveraging a result due to \citet{garivier2011kl}. Observe that this lemma (counter-intuitively) applies $d(\p,\q)$ to non-Bernoulli distributions.

\begin{lemma}[Chernoff bound following Corollary 10.4 and Note 10.1 of \citet{lattimore2020bandit}]
	Let $X_1, \ldots X_n$ be a sequence of independent and identically distributed random variables for which $X_t \in [0,1]$. Let $\mu$ denote each variable's expected value, and let $\hat{\mu} = \frac{1}{n}\sum_{t=1}^n X_t$ be the sample mean. For $v \geq 0$,
%	$$P(d(\hat{r}, r) \geq v, \hat{r} \leq r) \leq \exp(-nv),$$
%	$$P(d(\hat{r}, r) \geq v, \hat{r} \geq r) \leq \exp(-nv).$$
%	Furthermore, 
	define
	$U'(v) = \max\{\q \in [0, 1] : d(\hat{\mu}, \q) \leq v\}$ and
	$L'(v) = \min\{\q \in [0, 1] : d(\hat{\mu}, \q) \leq v\}.$
	Then $P(\mu \geq U'(v)) \leq \exp(-nv)$ and $P(\mu \leq L'(v)) \leq \exp(-nv)$.
	\label{lemma:Chernoff}
\end{lemma}

\begin{lemma}
If \GG holds, then with probability at least $1-\delta$, $s \geq PG$.
\label{lemma:GG}
\end{lemma}
\begin{proof}
Each $s_i$ is a independent and identically distributed random variable on the interval $[0,1]$; $s$ is the mean of the distribution from which each $s_i$ was sampled, and $\hat{s}$ is our sample estimate of $s$. So, from Lemma~\ref{lemma:Chernoff}, 
$$P(s \leq L'(v)) \leq \exp(-nv).$$
The event $s > L'(v)$ is complementary to the event $s \leq L'(v)$ so $$P(s > L'(v)) \geq 1 - \exp(-nv).$$
Substituting in $v = \ln(1/(\delta))/n$ and simplifying,
\begin{equation}
	P(s > L'(\ln(1/\delta)/n)) \geq 1-\delta.
	\label{eq:chernoff-L'}
\end{equation}
Appealing to Equation~(\ref{eq:L}), we can rewrite $L'(\ln(1/\delta)/n) = L(\hat{s},n,\delta)$. Then, Inequality~(\ref{eq:chernoff-L'}) tells us that with probability no less than $1-\delta$,
\begin{equation}
	s > L(\hat{s},n,\delta).
	\label{eq:s-gt-L}
\end{equation}
If \GG holds,
\begin{equation}
	L(\hat{s},n,\delta) \geq PG.
	\label{eq:GG}
\end{equation}
Intersecting Inequalities~(\ref{eq:s-gt-L}) and (\ref{eq:GG}) we obtain that with probability at least $1-\delta$, $s > PG.$ 
Thus, with probability at least $1-\delta$, $s \geq PG.$
\end{proof}

\begin{lemma}
If \BG holds, then with probability at least $1-\delta/2$, $s < PG.$
\label{lemma:BG}
\end{lemma}
\begin{proof}
	From Lemma~\ref{lemma:Chernoff}, 
	$$P(s \geq U'(v)) \leq \exp(-nv).$$
	The event $s < U'(v)$ is complementary to the event $s \geq U'(v)$ so $$P(s < U'(v)) \geq 1 - \exp(-nv).$$
	Substituting in $v = \ln(1/(\delta/2))/n$ and simplifying,
	\begin{equation}
		P(s < U'(\ln(1/(\delta/2))/n)) \geq 1-\delta/2.
		\label{eq:chernoff-U'}
	\end{equation}
	Appealing to Equation~(\ref{eq:U}), we can rewrite $U'(\ln(1/(\delta/2))/n) = U(\hat{s},n,\delta/2)$. Then, Inequality~(\ref{eq:chernoff-U'}) tells us that with probability no less than $1-\delta/2$,
	\begin{equation}
		s < U(\hat{s},n,\delta/2).
		\label{eq:s-lt-U}
	\end{equation}
	If \BG holds,
	\begin{equation}
		U(\hat{s},n,\delta/2) < PG
		\label{eq:BG}
	\end{equation}
	Intersecting Inequalities~(\ref{eq:s-lt-U}) and (\ref{eq:BG}) we obtain that with probability at least $1-\delta/2$, $s < PG.$ %Thus, with probability at least $1-\delta$, $s \geq PG.$
\end{proof}

\begin{lemma}
If \GR holds, then with probability at least $1-\delta$, $r \leq RID$. 
\label{lemma:GR}
\end{lemma}
\begin{proof}
Each measurement of whether or not the agent gave a ridiculous answer, denoted $r_i$, is an identically distributed Bernoulli random variable (each is either $1$ or $0$ depending on whether or not the agent received score zero for its answer to the sampled question); $r$ is the mean of the Bernoulli distribution from which each $r_i$ was sampled (the probability according to which the agent gives a ridiculous answer), and $\hat{r}$ is our sample estimate of $r$. We can therefore make an argument just as in Lemma~\ref{lemma:BG}, concluding analogously to Inequality~\eqref{eq:s-lt-U} that with probability no less than $1-\delta$,
\begin{equation}
	r < U(\hat{r},n,\delta).
\label{eq:r-lt-U}
\end{equation}
If \GR holds,
\begin{equation}
	U(\hat{r},n,\delta) \leq RID.
\label{eq:GR}
\end{equation}
Intersecting Inequalities~(\ref{eq:r-lt-U}) and (\ref{eq:GR}) we obtain that with probability at least $1-\delta$, $r < RID.$ Thus, with probability at least $1-\delta$, $r \leq RID.$
\end{proof}

\begin{lemma}
If \BR holds, then with probability at least $1-\delta/2$, $r > RID$.
\label{lemma:BR}
\end{lemma}

\begin{proof}
Likewise, we can make an argument just as in Lemma~\ref{lemma:GG}, concluding analogously to Inequality~\eqref{eq:s-gt-L} that with probability no less than $1-\delta/2$,
\begin{equation}
	r > L(\hat{r},n,\delta/2).
	\label{eq:r-gt-L}
\end{equation}
If \BR holds,
\begin{equation}
	L(\hat{r},n,\delta/2) > RID.
	\label{eq:BR}
\end{equation}
Intersecting Inequalities~(\ref{eq:r-gt-L}) and (\ref{eq:BR}) we obtain that with probability at least $1-\delta/2$, $r > RID.$
\end{proof}

We are now ready to prove Theorem~\ref{thm:samples}.

\setcounterref{theorem}{thm:samples}
\addtocounter{theorem}{-1}

\begin{theorem}
	If Procedure~\ref{def:testing-samples} determines that an agent does or does not understand a given scope, this conclusion is correct with probability at least $1-\delta$.
\end{theorem}

\begin{proof}
%First, note that there is no way that \GG and \BG can simultaneously hold, or likewise that \GR and \BR can simultaneously hold, and hence there is no way the testing procedure could determine both that the agent understands and that it does not. We thus treat each case in turn.
%
To show that an agent understands a given scope, we must demonstrate that both $s \geq PG$ and $r \leq RID$.
By Lemma~\ref{lemma:GG}, if \GG holds, $s < PG$ with probability at most $\delta$. Likewise, by Lemma~\ref{lemma:GR}, if \GR holds then $r > RID$ with probability at most $\delta$. 
%By the union bound, the probability that at least one of these two failures occurs is no more than the sum of the probabilities that each of them occurs individually.\footnote{Thus, it was not necessary to make $\deltaPG = \deltaOR = \delta/2$; we could have picked any positive values for $\deltaPG$ and $\deltaOR$ that sum to $\delta$ or less.} 
We call the event where one of these conditions holds but the desired inequality on the corresponding true distribution does not hold a {failure}. The event where \emph{both} failures occur is the intersection of the events in which each failure occurs, and so its probability is no greater than the maximum of the two failure probabilities. Each of these probabilities is bounded by $\delta$, so we can conclude that, if both \GG and \GR both hold, $(s \geq PG$ and $r \leq RID)$ will also hold with probability no less than $1- \delta$.

To show that the agent does \emph{not} understand a given scope, it is sufficient to demonstrate either that $s < PG$ or that $r > RID$. 
Our procedure draws this conclusion if at least one of \BG and \BR holds.  
By Lemma~\ref{lemma:BG}, if \BG holds, then with probability at least $1-\delta/2$, $s < PG.$
By Lemma~\ref{lemma:BR}, if \BR holds, then with probability at least $1-\delta/2$, $r > RID$.
By the union bound, the probability that at least one of these two conditions is satisfied while the corresponding condition on the underlying true distribution does not (a failure) is no more than the sum of the probabilities that each condition fails individually.\footnote{Thus, it was not necessary to make $\deltaPG = \deltaOR = \delta/2$; we could have picked any positive values for $\deltaPG$ and $\deltaOR$ that sum to $\delta$ or less.}
Thus, if Procedure~\ref{def:testing-samples} determines that an agent does not understand a given scope, the condition $(s < PG$ or $r > RID)$ holds with probability at least $1-(\delta/2 + \delta/2) = 1-\delta$.
\end{proof}

%k4
\setcounterref{theorem}{thm:explanations}
\addtocounter{theorem}{-1}

\begin{theorem}
	If Procedure~\ref{def:testing-explanations} determines that an agent does or does not understand a given scope, this conclusion is correct with probability at least $1-\delta$.
\end{theorem}

\begin{proof}
When we sample from $Q$, we obtain a sample from each $\Qpi$ with probability $\pai$ and from $\bA$ with probability $\bp = 1-\sum_i \pai$. First, let us consider average score. Denote the true score on $\bA$ as $\bs$.
Because the true score $s$ is the expectation over the scores of all questions in $Q$ and expectation is linear, $s$ is a convex combination of each $\sai$ and of $\bs$,
\begin{equation}
	s = \left( \sum_{\proc_i \in \Proc} \pai \cdot \sai \right) + \bp \cdot \bs.
	\label{eq:s-def}
\end{equation}
Consider each set of questions in turn. For questions from each $\Qpi$, our beliefs about the average score do not depend on our samples at all: we know that the agent's true average score is $\sai$.

The samples $\bs_1, \ldots, \bs_{\bn}$ we have taken from $\bA$ are independent random variables on $[0,1]$, so the argument in the proof of Lemma~\ref{lemma:GG} applies directly, substituting $n = \bn$, $\hat{s} = \hbs$ and $s = \bs$. Thus, with probability at least $1-\delta$, 
\begin{equation}
	\bs > L(\hbs,\bn,\delta).
\label{eq:GG2} 
\end{equation}
Intersecting Inequalities~\eqref{eq:s-def} and \eqref{eq:GG2}, with probability at least $1-\delta$,
\begin{equation}
	s > \left( \sum_{\proc_i \in \Proc} \pai \cdot \sai \right) + \bp \cdot L(\hbs,\bn,\delta)
	\label{eq:GG3}
\end{equation}
If \GGE holds, 
\begin{equation}
\left(\sum_{\proc_i \in \Proc} \pai \cdot \sai \right) + \bp \cdot L(\hbs,n,\delta) \geq PG
\label{eq:GGE}
\end{equation}
Intersecting Inequalities \eqref{eq:GG3} and \eqref{eq:GGE}, if \GGE holds then with probability at least $1-\delta$, $s > PG$.

We can make an analogous argument about \BGE. If \BGE holds,
\begin{equation}
\left(\sum_{\proc_i \in \Proc} \pai \cdot \sai \right) + \bp \cdot U(\hbs,\bn,\delta/2) < PG
\label{eq:BGE}
\end{equation}
As above, we can make the same argument as in the proof of Lemma~\ref{lemma:BG} but with $n = \bn$, $\hat{s} = \hbs$, and $s = \bs$, obtaining $s < U(\hat{s},n,\delta/2).$ Intersecting with Inequality~\eqref{eq:BGE}, we obtain that with probability at least $1-\delta/2$, $s < PG$.

Now let us consider the probability that the agent will return a ridiculous answer. Denote the true probability that the agent will return a ridiculous answer on $\bA$ as $\br$. By the Law of Total Probability, $r$ is a convex combination of each $\rai$ and of $\br$,
\begin{equation}
	r = \left( \sum_{\proc_i \in \Proc} \pai \cdot \rai \right) + \bp \cdot \br.
	\label{eq:r-def}
\end{equation}
As before, we can repeat the argument in the proof of Lemma~\ref{lemma:GR}, substituting $n = \bn$, $\hat{r} = \hbr$, and $r = \br$, obtaining $\br < U(\hbr,\bn,\delta).$ Intersecting with Inequality~\eqref{eq:r-def}, with probability at least $1-\delta/2$,
\begin{equation}
	r \leq \left( \sum_{\proc_i \in \Proc} \pai \cdot \rai \right) + \bp \cdot U(\hbr,\bn,\delta/2).
	\label{eq:r-def2}
\end{equation}
If \GRE holds,
\begin{equation}
\left(\sum_{\proc_i \in \Proc} \pai \cdot \rai \right) + \bp \cdot U(\hbr,\bn,\delta/2) \leq RID.
\label{eq:GRE}
\end{equation}
Intersecting Inequalities \eqref{eq:r-def2} and \eqref{eq:GRE}, if \GRE holds then with probability at least $1-\delta$, $r < RID$.

Finally, we can make an analogous argument about \BRE. If \BRE holds,
\begin{equation}
\left(\sum_{\proc_i \in \Proc} \pai \cdot \sai \right) + \bp \cdot L(\hbr,\bn,\delta) > RID
\label{eq:BRE}
\end{equation}
As above, we can make the same argument as in the proof of Lemma~\ref{lemma:BR}, substituting $n = \bn$, $\hat{r} = \hbr$, and $r = \br$, obtaining $\br > L(\hbr,\bn,\delta/2).$ Intersecting with Inequality~\eqref{eq:BRE}, we obtain that with probability at least $1-\delta/2$, $r > RID$.

We have now proven results exactly analogous to those of Lemmas~\ref{lemma:GG}, \ref{lemma:BG}, \ref{lemma:GR}, and \ref{lemma:BR} for \GGE, \BGE, \GRE, and \BGE, the four conditions used in Procedure~\ref{def:testing-samples}. We can thus follow the same argument as in the proof of Theorem~\ref{thm:samples} (substituting \GGE for \GG, etc., and appealing to these new results rather than each of the corresponding lemmas just mentioned) to show that if Procedure~\ref{def:testing-explanations} determines that an agent does or does not understand a given scope, this conclusion is correct with probability at least $1-\delta$.
\end{proof}

\includeuncertainty{
\setcounterref{theorem}{thm:explanations-uncertain}
\addtocounter{theorem}{-1}

\begin{theorem}
	If Procedure~\ref{def:testing-explanations-uncertain} determines that an agent does or does not understand a given scope, this conclusion is correct with probability at least $1-\delta$.
\end{theorem}
\vspace{-1.5em}
\newcommand{\snotpi}{s_{\neg\proc_i}}
\begin{proof}
First, consider average score. As in the case where application of trusted procedures is certain, when we sample from $Q$, we obtain a sample from each $\Qpi$ with probability $\pai$ and from $\bA$ with probability $\bp = 1-\sum_i \pai$. 
However, in the setting we now consider, when facing a question from $\Qpi$, the agent only uses procedure $\proc_i$ with some unknown probability $\piai$, for which we have empirical estimate $\hpiai$. Furthermore, when the agent faces a question from $\Qpi$ and \emph{does not} use $\proc_i$, it realizes some (unknown) expected score $\snotpi$.
Because the true score $s$ is the expectation over the scores of all questions in $Q$ and expectation is linear, $s$ is the following convex combination:
\begin{equation}
	s = \left( \sum_{\proc_i \in \Proc} \pai \cdot \left(\piai \cdot \sai + (1-\piai) \cdot \snotpi \right) \right) + \bp \cdot \bs.
	\label{eq:s-u-def}
\end{equation}
Each score $\snotpi$ is no less than 0, so we can substitute $\snotpi=0$ into Equation~\eqref{eq:s-u-def} to obtain
\begin{equation}
	s \geq \left( \sum_{\proc_i \in \Proc} \pai \cdot \piai \cdot \sai\right) + \bp \cdot \bs.
	\label{eq:s-u-def2}
\end{equation}
We can bound the last term in Inequality~\eqref{eq:s-u-def2} in exactly the same way as in the argument leading up to Inequality~\eqref{eq:GG2}, so with probability at least $1-\delta'$,
\begin{equation}
	\bs > L(\hbs,\bn,\delta').
	\label{eq:GGU-bs} 
\end{equation}
Making the same argument as in Lemma~\ref{lemma:BR} we can derive an analogue of Inequality~\eqref{eq:r-gt-L}	: for each $\proc_i$, with probability at least $1-\delta'$, 
\begin{equation}
	\piai > L(\hpiai,\nai,\delta').
	\label{eq:piai}
\end{equation}
We could intersect any instance of Inequality~\eqref{eq:piai} (corresponding to a single $\proc_i$) with Inequality~\eqref{eq:s-u-def2} to obtain a weaker bound on $s$ that holds with probability at least $1-\delta'$; this amounts to shifting probability mass away from the fraction of questions receiving score $\sai$ and towards the fraction receiving score $0$. Likewise we could intersect Inequality~\eqref{eq:GGU-bs} with Inequality~\eqref{eq:s-u-def2}, obtaining a bound that holds with probability at least $1-\delta'$.
By the union bound, the probability that at least one of these lower bounds fails is upper bounded by the sum of their failure probabilities. There are $|\Proc|$ instances of Inequality~\eqref{eq:piai} and a single instance of Inequality~\eqref{eq:GGU-bs}, so let $\delta' = \dai$. Then, with probability at least $1-\delta$, 
\begin{equation}
	s \geq \left( \sum_{\proc_i \in \Proc} \pai \cdot L\left(\hpiai,\nai,\dai\right) \cdot \sai\right) + \bp \cdot L\left(\hbs,\bn,\dai\right).
	\label{eq:s-u-def3}
\end{equation}
If \GGU holds,
\begin{equation}
\label{eq:GGU}
\end{equation}
Intersecting Inequalities~\eqref{eq:s-u-def3} and \eqref{eq:GGU}, if \GGU holds then with probability at least $1-\delta$, $s \geq PG$.

For reasons that we discuss after this proof, we ignore demonstrations that the agent uses the procedure in assessing ridiculousness. Instead, Procedure~\ref{def:testing-explanations-uncertain} determines that the agent understands a given scope if both \GGU and \GR hold, where \GR is the same condition used in Procedure~\ref{def:testing-samples}. Lemma~\ref{lemma:GR} thus applies directly: if \GR holds, with probability at least $1-\delta$, $r \leq RID$. Following the argument in the proof of Theorem~\ref{thm:samples}, the probability that both \GGU and \GR hold but that $s < PG$ and/or $r > RID$ is no greater than the probability that either \GGU or \GR fails independently. We have just shown that each condition fails with probability no greater than $1-\delta$, so if both \GG and \GR hold, the agent understands the given scope with probability at least $1-\delta$.

Now let us consider the scenario where Procedure~\ref{def:testing-explanations-uncertain} determines that the agent does not understand.
Each score $\snotpi$ is no greater than 1, so we can substitute $\snotpi=1$ into Equation~\eqref{eq:s-u-def} to obtain
\begin{equation}
	s \leq \left( \sum_{\proc_i \in \Proc} \pai \cdot \left(\piai \cdot \sai + (1-\piai) \right) \right) + \bp \cdot \bs.
	\label{eq:s-u-def5}
\end{equation}
We can bound the last term in Inequality~\eqref{eq:s-u-def5} in exactly the same way as in the argument leading up to Inequality~\eqref{eq:GG2}, so with probability at least $1-\delta'$,
\begin{equation}
	\bs < U(\hbs,\bn,\delta').
	\label{eq:GRU-bs2} 
\end{equation}Making the same argument as in Lemma~\ref{lemma:GR} we can derive an analogue of Inequality~\eqref{eq:r-lt-U}: for each $\proc_i$, with probability at least $1-\delta'$, 
\begin{equation}
	\piai < U(\hpiai,\nai,\delta').
	\label{eq:piai2}
\end{equation}

We can intersect any instance of Inequality~\eqref{eq:piai2} with Inequality~\eqref{eq:s-u-def3} or intersect Inequality~\eqref{eq:GGU-bs} with Inequality~\eqref{eq:s-u-def3}, in either case obtaining a bound that holds with probability at least $1-\delta'$.
By the union bound, the probability that at least one of these lower bounds fails is upper bounded by the sum of their failure probabilities. There are $|\Proc|$ instances of Inequality~\eqref{eq:piai2} and a single instance of Inequality~\eqref{eq:GRU-bs2}. We will reserve half of our failure probability for our second test \BR, so let $\delta' =  \daii$. Then, with probability at least $1-\delta/2$, 
\begin{multline}
	s \leq \left( \sum_{\proc_i \in \Proc} \pai \cdot \left(U\left(\hpiai,\nai,\daii\right) \cdot \sai + \left(1-U\left(\hpiai,\nai,\daii\right)\right) \right) \right) \\ + \bp \cdot U\left(\hbs,\bn,\daii\right).
	\label{eq:s-u-def4}
\end{multline}
If \BGU holds,
\begin{multline}
\left( \sum_{\proc_i \in \Proc} \pai \cdot \left( U\left(\hpiai,n_{\proc_i},\daii\right) \cdot s_\alpha + \left(1-U\left(\hpiai,\nai,\daii\right)\right)\right)
\right) \\ +\, \bp \cdot U\left(\hbs,\bn,\daii\right) \leq PG	\label{eq:BGU}
\end{multline}
Intersecting Inequalities~\eqref{eq:s-u-def4} and \eqref{eq:BGU}, if \BGU holds then with probability at least $1-\delta/2$, $s \leq PG$.

Procedure~\ref{def:testing-explanations-uncertain} determines that the agent does not understand a given scope if either \BGU or \BR holds, where \BR is the same condition used in Procedure~\ref{def:testing-samples}. Lemma~\ref{lemma:BR} thus applies directly: if \BR holds, with probability at least $1-\delta/2$, $r \geq RID$. By the union bound, the probability that at least one of \GGU and \GR holds but that $(s \geq PG$ and $r \leq RID)$ is no greater than the sum of the probabilities that \GGU and \GR fail independently, $\delta = \delta/2 + \delta/2$. Thus, if either \BGU or \BR holds, the agent fails to understand the given scope with probability at least $1-\delta$.
\end{proof}

\section{On Testing Ridiculousness with Uncertain Explanations}
\label{app:uncertainty-ridiculousness}

Why does Procedure~\ref{def:testing-explanations-uncertain} use \GR and \BR rather than new tests \GRU and \BRU analogous to \GGU and \BGU? Of course, we could derive such tests and they would indeed be bounds on $r$; however, they would be looser bounds than \GR and \BR in the common case where the probability of giving a ridiculous answer is very close to zero. Let us see why. 
First, let us state \BR:
\begin{equation}
	U(0,n,\delta) \leq RID.	
\end{equation}
Now let us consider \GRU; \BRU is analogous. The test we would derive is
\begin{multline}
	\left( \sum_{\proc_i \in \Proc} \pai \cdot \left( L\left(\hpiai,n_{\proc_i},\dai\right) \cdot r_\alpha + \left(1-L\left(\hpiai,\nai,\dai\right)\right)\right)
\right) \\ +\, \bp \cdot U\left(\hbr,\bn,\dai\right) \leq RID
\label{eq:GRU1}
\end{multline}
For simplicity, let us analyze the case where there is only one procedure $\proc$. Let us consider the case where $r_\proc = 0$ (using the procedure never generates a ridiculous answer), $\hpia = 1$ (the procedure is always observed to be used), and $\hbr = 0$ (no ridiculous answers are ever observed). 
This, or something close to it, is what we should expect as the common case for the application of our tests in practice. Specialized to this case, Inequality~\eqref{eq:GRU1} states
\begin{equation}
	\pa \cdot \left(1-L\left(1,\na,\delta/2\right)\right)
	 + (1-\pa) \cdot U\left(0,\bn,\delta/2\right) \leq RID
	\label{eq:GRU2}
\end{equation}
For every $\p,n,\delta$, it is always the case that $1-L(\p,n,\delta) = U(1-\p,n,\delta)$, due to symmetry in the definitions of relative entropy for Bernoulli random variables and $L$ and $U$. 
(Intuitively:  $1-L$ gives the same answer $U$ would have given if the Bernoulli variable's outcomes are relabelled so that $0$ becomes $1$ and vice versa.) 
So Inequality~\eqref{eq:GRU2} can be rewritten as
\begin{equation}
	\pa \cdot U\left(0,\na,\delta/2\right)
	+ (1-\pa) \cdot U\left(0,\bn,\delta/2\right) \leq RID
	\label{eq:GRU3}
\end{equation}
We know that $n = \na + \bn$. Assuming that neither $\na$ nor $\bn$ is zero, $n>\na$ and $n > \bn$. $U$ is strictly monotonically decreasing in its second parameter (the bounds tighten as $n$ increases), so 
\begin{equation}
	\pa \cdot U(0,n,\delta/2) + (1-\pa) \cdot U(0,n,\delta/2) <
	\pa \cdot U(0,\na,\delta/2) + (1-\pa) \cdot U(0,\bn,\delta/2) 
\end{equation}
Simplifying,
\begin{equation}
	U(0,n,\delta/2) < \pa \cdot U(0,\na,\delta/2) + (1-\pa) \cdot U(0,\bn,\delta/2) 
\end{equation}
Furthermore, $U$ is also monotonically decreasing in its third parameter (the bounds tighten as the failure probability grows). Thus,
\begin{equation}
	U(0,n,\delta) < U(0,n,\delta/2) < \pa \cdot U(0,\na,\delta/2) + (1-\pa) \cdot U(0,\bn,\delta/2) 
\end{equation}
We have just shown that even when there is only one procedure $\proc$,  if $r_\proc = 0$, $\hpia = 1$, and $\hbr = 0$, the bound in \GR is tighter than the bound in \GRU.
}

\section{On Chernoff vs Hoeffding Bounds}
\label{app:Hoeffding-Chernoff}

In our testing procedures, we leveraged upper- and lower-bounding functions $U$ and $L$ derived from the Chernoff bound, which lack closed-form expressions and hence require numerical approximation. However, we could alternatively have derived our four tests of average grade and ridiculousness via the Hoeffding bound, which does have a closed-form expression. Specifically,
each $s_i$ is an identically distributed random variable restricted to the interval $[0,1]$ and each $r_i$ is a Bernoulli random variable. We could thus have used Hoeffding's inequality to bound the gap between empirical and true distributions in both cases. Consider the simplest setting of Procedure~\ref{def:testing-samples}. Here we could have concluded e.g., that $$\hat{s} - \sqrt{\frac{-\ln(\deltaPG)}{n}} \leq s$$ with probability at least $1-\deltaPG$ and similarly that $$\hat{r} +\sqrt{\frac{-\ln(\deltaOR/2)}{n}} \geq r$$ with probability at least $1-\deltaOR/2$. 

\begin{table}[t]
\begin{tabular}{rcccc}
\toprule
$n$ & $0.9 - \sqrt{\frac{-\ln(0.05)}{n}}$ & $0.9 + \sqrt{\frac{-\ln(0.025)}{n}}$ & $0.5 - \sqrt{\frac{-\ln(0.05)}{n}}$ & $0.5 + \sqrt{\frac{-\ln(0.025)}{n}}$ \\
\midrule
10 & 0.352667 & 1.507361 & -0.047333 & 1.107361 \\
100 & 0.726918 & 1.092065 & 0.326918 & 0.692065 \\
1,000 & 0.845267 & 0.960736 & 0.445267 & 0.560736 \\
10,000 & 0.882692 & 0.919206 & 0.482692 & 0.519206 \\
100,000 & 0.894527 & 0.906074 & 0.494527 & 0.506074 \\
1,000,000 & 0.898269 & 0.901921 & 0.498269 & 0.501921 \\
10,000,000 & 0.899453 & 0.900607 & 0.499453 & 0.500607 \\
\bottomrule
\end{tabular}
\vspace{.6em}	
\caption{Example values for bounding ${s}$ in Procedure~\ref{def:testing-samples}, $\delta = 0.05$, leveraging the Hoeffding bound instead of the Chernoff bound. Compare to corresponding values in Table~\ref{table:bounding-s-Chernoff}. As in that table, the two pairs of columns represent the confidence intervals around observed empirical averages of $\hat{s} = 0.9$ and $\hat{s} = 0.5$ respectively. Observe that the width of the confidence interval derived from the Hoeffding bound is the same in both cases. When values are greater than 1 or less than 0, the Hoeffding bound is uninformative.
}
\label{table:bounding-s-Hoeffding}
\end{table}

\begin{table}[t] % this may end up going in an appendix anyway
%	\begin{minipage}{.49\textwidth}\centering
	\centering
	\begin{tabular}{rcccccc}
		\toprule
$n$ & $0 + \sqrt{\frac{-\ln(0.05)}{n}}$ & $3/n + \sqrt{\frac{-\ln(0.05)}{n}}$ & $0.01 - \sqrt{\frac{-\ln(0.025)}{n}}$ \\
\midrule
10 & 0.547333 & 0.847333 & -0.597361 \\
100 & 0.173082 & 0.203082 & -0.182065 \\
1,000 & 0.054733 & 0.057733 & -0.050736 \\
10,000 & 0.017308 & 0.017608 & -0.009206 \\
100,000 & 0.005473 & 0.005503 & 0.003926 \\
1,000,000 & 0.001731 & 0.001734 & 0.008079 \\
10,000,000 & 0.000547 & 0.000548 & 0.009393 \\
		\bottomrule
	\end{tabular}
	%	\end{minipage}
\vspace{.6em}	
\caption{Example values for bounding $r$ in Procedure~\ref{def:testing-samples} with $\delta = 0.05$, leveraging the Hoeffding bound instead of the Chernoff bound. Compare to corresponding values in Table~\ref{table:bounding-r-Chernoff}. When values are greater than 1 or less than 0, the Hoeffding bound is uninformative.
}
\label{table:bounding-r-Hoeffding}
\end{table}

However, tests based on these bounds are somewhat worse for tests of average grade and lamentably difficult to satisfy for tests of ridiculousness. See Tables~\ref{table:bounding-s-Hoeffding} and \ref{table:bounding-r-Hoeffding}, which are analogues of Tables~\ref{table:bounding-s-Chernoff} and \ref{table:bounding-r-Chernoff} using the Hoeffding bound. 
Consider a setting in which the agent has given no ridiculous answers. How many samples would we need to establish confidence that the true probability of giving a ridiculous answer has fallen below $0.00052$? We can infer the answer from Table~\ref{table:bounding-r-Hoeffding}: over ten million samples. In contrast, our test derived from the Chernoff bound required fewer than 10,000 samples.
		
Why do the two inequalities produce such different bounds, particularly when testing ridiculousness? The issue is that $RID$ is very close to zero and, in the event that our test passes, $\hat{r}$ will be too. This is a problem for the Hoeffding bound when applied to Bernoulli random variables, which turns out to be extremely loose when the true probability is close to 0 or 1. Indeed, the Hoeffding bound always scales with the square root of the number of samples. The Chernoff bound does the same when the true probability is 0.5, but scales almost linearly with the number of samples when the true probability is close to 0 or 1.

\section{Wrong Answers and Bogus Explanations: An Example}
\label{TemporalExample}
Below is a session with ChatGPT on November 19th, 2023. ChatGPT provides wrong answers, and bogus explanations of the answers. (This is not a critique specifically of ChatGPT, which is used simply as an example.)

\texttt{
\\
\textbf{You}\\
Consider events A, B, C. B happened during A. C started after A did, and ended before B did. Which started first - B or C?\\
\\
\textbf{ChatGPT}\\
Let's break down the information:\\
•	Event A started.\\
•	Event B happened during A.\\
•	Event C started after A did.\\
•	Event C ended before B did.\\
Given this information, we can conclude that:\\
1.	A started first.\\
2.	C started after A.\\
3.	C ended before B.\\
Therefore, in the sequence of events, C started after A, and since C ended before B, it means that C started after B as well. The chronological order would be B, C, and then A. \\
 \\
\textbf{You}\\
Are you sure? Remember that events have duration.\\
\\
\textbf{ChatGPT}\\
Apologies for any confusion in my previous responses. Let's consider the duration of events:\\
•	Event A started.\\
•	Event B happened during A.\\
•	Event C started after A did and ended before B did.\\
Given the duration of events, we can refine the chronological order:\\
1.	A started.\\
2.	C started after A and ended before B.\\
So, considering the duration, C started after A and ended before B. The correct chronological order is A, C, and then B. I appreciate your patience, and thank you for pointing out the importance of event duration. If you have any more questions or need further clarification, feel free to ask.\\
ChatGPT can make mistakes. Consider checking important information.
}

\shortver{
\newpage
\section*{NeurIPS Paper Checklist}

\begin{enumerate}
	
	\item {\bf Claims}
	\item[] Question: Do the main claims made in the abstract and introduction accurately reflect the paper's contributions and scope?
	\item[] Answer: \answerYes{} % \answerTODO{} % Replace by \answerYes{}, \answerNo{}, or \answerNA{}.
	\item[] Justification: Claims are backed up by theorems. % \justificationTODO{}
	\item[] Guidelines:
	\begin{itemize}
		\item The answer NA means that the abstract and introduction do not include the claims made in the paper.
		\item The abstract and/or introduction should clearly state the claims made, including the contributions made in the paper and important assumptions and limitations. A No or NA answer to this question will not be perceived well by the reviewers. 
		\item The claims made should match theoretical and experimental results, and reflect how much the results can be expected to generalize to other settings. 
		\item It is fine to include aspirational goals as motivation as long as it is clear that these goals are not attained by the paper. 
	\end{itemize}
	
	\item {\bf Limitations}
	\item[] Question: Does the paper discuss the limitations of the work performed by the authors?
	\item[] Answer: \answerYes{} % \answerTODO{} % Replace by \answerYes{}, \answerNo{}, or \answerNA{}.
	\item[] Justification: Limitations are discussed throughout, but particularly in the Discussion section. % \justificationTODO{}
	\item[] Guidelines:
	\begin{itemize}
		\item The answer NA means that the paper has no limitation while the answer No means that the paper has limitations, but those are not discussed in the paper. 
		\item The authors are encouraged to create a separate "Limitations" section in their paper.
		\item The paper should point out any strong assumptions and how robust the results are to violations of these assumptions (e.g., independence assumptions, noiseless settings, model well-specification, asymptotic approximations only holding locally). The authors should reflect on how these assumptions might be violated in practice and what the implications would be.
		\item The authors should reflect on the scope of the claims made, e.g., if the approach was only tested on a few datasets or with a few runs. In general, empirical results often depend on implicit assumptions, which should be articulated.
		\item The authors should reflect on the factors that influence the performance of the approach. For example, a facial recognition algorithm may perform poorly when image resolution is low or images are taken in low lighting. Or a speech-to-text system might not be used reliably to provide closed captions for online lectures because it fails to handle technical jargon.
		\item The authors should discuss the computational efficiency of the proposed algorithms and how they scale with dataset size.
		\item If applicable, the authors should discuss possible limitations of their approach to address problems of privacy and fairness.
		\item While the authors might fear that complete honesty about limitations might be used by reviewers as grounds for rejection, a worse outcome might be that reviewers discover limitations that aren't acknowledged in the paper. The authors should use their best judgment and recognize that individual actions in favor of transparency play an important role in developing norms that preserve the integrity of the community. Reviewers will be specifically instructed to not penalize honesty concerning limitations.
	\end{itemize}
	
	\item {\bf Theory Assumptions and Proofs}
	\item[] Question: For each theoretical result, does the paper provide the full set of assumptions and a complete (and correct) proof?
	\item[] Answer: \answerYes{} % \answerTODO{} % Replace by \answerYes{}, \answerNo{}, or \answerNA{}.
	\item[] Justification: The formal model is described, and a full proof is given for each theoretical result.  % \justificationTODO{}
	\item[] Guidelines:
	\begin{itemize}
		\item The answer NA means that the paper does not include theoretical results. 
		\item All the theorems, formulas, and proofs in the paper should be numbered and cross-referenced.
		\item All assumptions should be clearly stated or referenced in the statement of any theorems.
		\item The proofs can either appear in the main paper or the supplemental material, but if they appear in the supplemental material, the authors are encouraged to provide a short proof sketch to provide intuition. 
		\item Inversely, any informal proof provided in the core of the paper should be complemented by formal proofs provided in appendix or supplemental material.
		\item Theorems and Lemmas that the proof relies upon should be properly referenced. 
	\end{itemize}
	
	\item {\bf Experimental Result Reproducibility}
	\item[] Question: Does the paper fully disclose all the information needed to reproduce the main experimental results of the paper to the extent that it affects the main claims and/or conclusions of the paper (regardless of whether the code and data are provided or not)?
	\item[] Answer: \answerNA{} % \answerTODO{} % Replace by \answerYes{}, \answerNo{}, or \answerNA{}.
	\item[] Justification: The paper does not include experiments. % \justificationTODO{}
	\item[] Guidelines:
	\begin{itemize}
		\item The answer NA means that the paper does not include experiments.
		\item If the paper includes experiments, a No answer to this question will not be perceived well by the reviewers: Making the paper reproducible is important, regardless of whether the code and data are provided or not.
		\item If the contribution is a dataset and/or model, the authors should describe the steps taken to make their results reproducible or verifiable. 
		\item Depending on the contribution, reproducibility can be accomplished in various ways. For example, if the contribution is a novel architecture, describing the architecture fully might suffice, or if the contribution is a specific model and empirical evaluation, it may be necessary to either make it possible for others to replicate the model with the same dataset, or provide access to the model. In general. releasing code and data is often one good way to accomplish this, but reproducibility can also be provided via detailed instructions for how to replicate the results, access to a hosted model (e.g., in the case of a large language model), releasing of a model checkpoint, or other means that are appropriate to the research performed.
		\item While NeurIPS does not require releasing code, the conference does require all submissions to provide some reasonable avenue for reproducibility, which may depend on the nature of the contribution. For example
		\begin{enumerate}
			\item If the contribution is primarily a new algorithm, the paper should make it clear how to reproduce that algorithm.
			\item If the contribution is primarily a new model architecture, the paper should describe the architecture clearly and fully.
			\item If the contribution is a new model (e.g., a large language model), then there should either be a way to access this model for reproducing the results or a way to reproduce the model (e.g., with an open-source dataset or instructions for how to construct the dataset).
			\item We recognize that reproducibility may be tricky in some cases, in which case authors are welcome to describe the particular way they provide for reproducibility. In the case of closed-source models, it may be that access to the model is limited in some way (e.g., to registered users), but it should be possible for other researchers to have some path to reproducing or verifying the results.
		\end{enumerate}
	\end{itemize}

	\item {\bf Open access to data and code}
	\item[] Question: Does the paper provide open access to the data and code, with sufficient instructions to faithfully reproduce the main experimental results, as described in supplemental material?
	\item[] Answer: \answerNA{} %\answerTODO{} % Replace by \answerYes{}, \answerNo{}, or \answerNA{}.
	\item[] Justification: The paper does not include experiments requiring code. % \justificationTODO{}
	\item[] Guidelines:
	\begin{itemize}
		\item The answer NA means that paper does not include experiments requiring code.
		\item Please see the NeurIPS code and data submission guidelines (\url{https://nips.cc/public/guides/CodeSubmissionPolicy}) for more details.
		\item While we encourage the release of code and data, we understand that this might not be possible, so “No” is an acceptable answer. Papers cannot be rejected simply for not including code, unless this is central to the contribution (e.g., for a new open-source benchmark).
		\item The instructions should contain the exact command and environment needed to run to reproduce the results. See the NeurIPS code and data submission guidelines (\url{https://nips.cc/public/guides/CodeSubmissionPolicy}) for more details.
		\item The authors should provide instructions on data access and preparation, including how to access the raw data, preprocessed data, intermediate data, and generated data, etc.
		\item The authors should provide scripts to reproduce all experimental results for the new proposed method and baselines. If only a subset of experiments are reproducible, they should state which ones are omitted from the script and why.
		\item At submission time, to preserve anonymity, the authors should release anonymized versions (if applicable).
		\item Providing as much information as possible in supplemental material (appended to the paper) is recommended, but including URLs to data and code is permitted.
	\end{itemize}

	\item {\bf Experimental Setting/Details}
	\item[] Question: Does the paper specify all the training and test details (e.g., data splits, hyperparameters, how they were chosen, type of optimizer, etc.) necessary to understand the results?
	\item[] Answer: \answerNA{} % \answerTODO{} % Replace by \answerYes{}, \answerNo{}, or \answerNA{}.
	\item[] Justification:  The paper does not include experiments. %\justificationTODO{}
	\item[] Guidelines:
	\begin{itemize}
		\item The answer NA means that the paper does not include experiments.
		\item The experimental setting should be presented in the core of the paper to a level of detail that is necessary to appreciate the results and make sense of them.
		\item The full details can be provided either with the code, in appendix, or as supplemental material.
	\end{itemize}
	
	\item {\bf Experiment Statistical Significance}
	\item[] Question: Does the paper report error bars suitably and correctly defined or other appropriate information about the statistical significance of the experiments?
	\item[] Answer: \answerNA{} % \answerTODO{} % Replace by \answerYes{}, \answerNo{}, or \answerNA{}.
	\item[] Justification:  The paper does not include experiments. % \justificationTODO{}
	\item[] Guidelines:
	\begin{itemize}
		\item The answer NA means that the paper does not include experiments.
		\item The authors should answer "Yes" if the results are accompanied by error bars, confidence intervals, or statistical significance tests, at least for the experiments that support the main claims of the paper.
		\item The factors of variability that the error bars are capturing should be clearly stated (for example, train/test split, initialization, random drawing of some parameter, or overall run with given experimental conditions).
		\item The method for calculating the error bars should be explained (closed form formula, call to a library function, bootstrap, etc.)
		\item The assumptions made should be given (e.g., Normally distributed errors).
		\item It should be clear whether the error bar is the standard deviation or the standard error of the mean.
		\item It is OK to report 1-sigma error bars, but one should state it. The authors should preferably report a 2-sigma error bar than state that they have a 96\% CI, if the hypothesis of Normality of errors is not verified.
		\item For asymmetric distributions, the authors should be careful not to show in tables or figures symmetric error bars that would yield results that are out of range (e.g. negative error rates).
		\item If error bars are reported in tables or plots, The authors should explain in the text how they were calculated and reference the corresponding figures or tables in the text.
	\end{itemize}
	
	\item {\bf Experiments Compute Resources}
	\item[] Question: For each experiment, does the paper provide sufficient information on the computer resources (type of compute workers, memory, time of execution) needed to reproduce the experiments?
	\item[] Answer: \answerNA{} %\answerTODO{} % Replace by \answerYes{}, \answerNo{}, or \answerNA{}.
	\item[] Justification:  The paper does not include experiments. % \justificationTODO{}
	\item[] Guidelines:
	\begin{itemize}
		\item The answer NA means that the paper does not include experiments.
		\item The paper should indicate the type of compute workers CPU or GPU, internal cluster, or cloud provider, including relevant memory and storage.
		\item The paper should provide the amount of compute required for each of the individual experimental runs as well as estimate the total compute. 
		\item The paper should disclose whether the full research project required more compute than the experiments reported in the paper (e.g., preliminary or failed experiments that didn't make it into the paper). 
	\end{itemize}
	
	\item {\bf Code Of Ethics}
	\item[] Question: Does the research conducted in the paper conform, in every respect, with the NeurIPS Code of Ethics \url{https://neurips.cc/public/EthicsGuidelines}?
	\item[] Answer: \answerYes{} % \answerTODO{} % Replace by \answerYes{}, \answerNo{}, or \answerNA{}.
	\item[] Justification: We have read the code of ethics. % \justificationTODO{}
	\item[] Guidelines:
	\begin{itemize}
		\item The answer NA means that the authors have not reviewed the NeurIPS Code of Ethics.
		\item If the authors answer No, they should explain the special circumstances that require a deviation from the Code of Ethics.
		\item The authors should make sure to preserve anonymity (e.g., if there is a special consideration due to laws or regulations in their jurisdiction).
	\end{itemize}

	\item {\bf Broader Impacts}
	\item[] Question: Does the paper discuss both potential positive societal impacts and negative societal impacts of the work performed?
	\item[] Answer: \answerYes{} % \answerTODO{} % Replace by \answerYes{}, \answerNo{}, or \answerNA{}.
	\item[] Justification: Given that our work has a philosophical bent, discussion of broader impact is woven through much of the paper, particular Sections 1, 2 and 6. % \justificationTODO{}
	\item[] Guidelines:
	\begin{itemize}
		\item The answer NA means that there is no societal impact of the work performed.
		\item If the authors answer NA or No, they should explain why their work has no societal impact or why the paper does not address societal impact.
		\item Examples of negative societal impacts include potential malicious or unintended uses (e.g., disinformation, generating fake profiles, surveillance), fairness considerations (e.g., deployment of technologies that could make decisions that unfairly impact specific groups), privacy considerations, and security considerations.
		\item The conference expects that many papers will be foundational research and not tied to particular applications, let alone deployments. However, if there is a direct path to any negative applications, the authors should point it out. For example, it is legitimate to point out that an improvement in the quality of generative models could be used to generate deepfakes for disinformation. On the other hand, it is not needed to point out that a generic algorithm for optimizing neural networks could enable people to train models that generate Deepfakes faster.
		\item The authors should consider possible harms that could arise when the technology is being used as intended and functioning correctly, harms that could arise when the technology is being used as intended but gives incorrect results, and harms following from (intentional or unintentional) misuse of the technology.
		\item If there are negative societal impacts, the authors could also discuss possible mitigation strategies (e.g., gated release of models, providing defenses in addition to attacks, mechanisms for monitoring misuse, mechanisms to monitor how a system learns from feedback over time, improving the efficiency and accessibility of ML).
	\end{itemize}
	
	\item {\bf Safeguards}
	\item[] Question: Does the paper describe safeguards that have been put in place for responsible release of data or models that have a high risk for misuse (e.g., pretrained language models, image generators, or scraped datasets)?
	\item[] Answer: \answerNA{} %\answerTODO{} % Replace by \answerYes{}, \answerNo{}, or \answerNA{}.
	\item[] Justification: the paper poses no such risks. % \justificationTODO{}
	\item[] Guidelines:
	\begin{itemize}
		\item The answer NA means that the paper poses no such risks.
		\item Released models that have a high risk for misuse or dual-use should be released with necessary safeguards to allow for controlled use of the model, for example by requiring that users adhere to usage guidelines or restrictions to access the model or implementing safety filters. 
		\item Datasets that have been scraped from the Internet could pose safety risks. The authors should describe how they avoided releasing unsafe images.
		\item We recognize that providing effective safeguards is challenging, and many papers do not require this, but we encourage authors to take this into account and make a best faith effort.
	\end{itemize}
	
	\item {\bf Licenses for existing assets}
	\item[] Question: Are the creators or original owners of assets (e.g., code, data, models), used in the paper, properly credited and are the license and terms of use explicitly mentioned and properly respected?
	\item[] Answer: \answerNA{} % \answerTODO{} % Replace by \answerYes{}, \answerNo{}, or \answerNA{}.
	\item[] Justification: The paper does not use existing assets. % \justificationTODO{}
	\item[] Guidelines:
	\begin{itemize}
		\item The answer NA means that the paper does not use existing assets.
		\item The authors should cite the original paper that produced the code package or dataset.
		\item The authors should state which version of the asset is used and, if possible, include a URL.
		\item The name of the license (e.g., CC-BY 4.0) should be included for each asset.
		\item For scraped data from a particular source (e.g., website), the copyright and terms of service of that source should be provided.
		\item If assets are released, the license, copyright information, and terms of use in the package should be provided. For popular datasets, \url{paperswithcode.com/datasets} has curated licenses for some datasets. Their licensing guide can help determine the license of a dataset.
		\item For existing datasets that are re-packaged, both the original license and the license of the derived asset (if it has changed) should be provided.
		\item If this information is not available online, the authors are encouraged to reach out to the asset's creators.
	\end{itemize}
	
	\item {\bf New Assets}
	\item[] Question: Are new assets introduced in the paper well documented and is the documentation provided alongside the assets?
	\item[] Answer: \answerNA{} % \answerTODO{} % Replace by \answerYes{}, \answerNo{}, or \answerNA{}.
	\item[] Justification: The paper does not release new assets. % \justificationTODO{}
	\item[] Guidelines:
	\begin{itemize}
		\item The answer NA means that the paper does not release new assets.
		\item Researchers should communicate the details of the dataset/code/model as part of their submissions via structured templates. This includes details about training, license, limitations, etc. 
		\item The paper should discuss whether and how consent was obtained from people whose asset is used.
		\item At submission time, remember to anonymize your assets (if applicable). You can either create an anonymized URL or include an anonymized zip file.
	\end{itemize}
	
	\item {\bf Crowdsourcing and Research with Human Subjects}
	\item[] Question: For crowdsourcing experiments and research with human subjects, does the paper include the full text of instructions given to participants and screenshots, if applicable, as well as details about compensation (if any)? 
	\item[] Answer: \answerNA{} % \answerTODO{} % Replace by \answerYes{}, \answerNo{}, or \answerNA{}.
	\item[] Justification: The paper did not conduct crowdsourcing nor research with human subjects. % \justificationTODO{}
	\item[] Guidelines:
	\begin{itemize}
		\item The answer NA means that the paper does not involve crowdsourcing nor research with human subjects.
		\item Including this information in the supplemental material is fine, but if the main contribution of the paper involves human subjects, then as much detail as possible should be included in the main paper. 
		\item According to the NeurIPS Code of Ethics, workers involved in data collection, curation, or other labor should be paid at least the minimum wage in the country of the data collector. 
	\end{itemize}
	
	\item {\bf Institutional Review Board (IRB) Approvals or Equivalent for Research with Human Subjects}
	\item[] Question: Does the paper describe potential risks incurred by study participants, whether such risks were disclosed to the subjects, and whether Institutional Review Board (IRB) approvals (or an equivalent approval/review based on the requirements of your country or institution) were obtained?
	\item[] Answer: \answerNA{} %\answerTODO{} % Replace by \answerYes{}, \answerNo{}, or \answerNA{}.
	\item[] Justification: the paper does not involve crowdsourcing nor research with human subjects. %\justificationTODO{}
	\item[] Guidelines:
	\begin{itemize}
		\item The answer NA means that the paper does not involve crowdsourcing nor research with human subjects.
		\item Depending on the country in which research is conducted, IRB approval (or equivalent) may be required for any human subjects research. If you obtained IRB approval, you should clearly state this in the paper. 
		\item We recognize that the procedures for this may vary significantly between institutions and locations, and we expect authors to adhere to the NeurIPS Code of Ethics and the guidelines for their institution. 
		\item For initial submissions, do not include any information that would break anonymity (if applicable), such as the institution conducting the review.
	\end{itemize}
	
\end{enumerate}
}

\end{document}